\documentclass[sigconf,balance=false, nonacm]{acmart}
\usepackage{popets}

\setcopyright{popets}
\copyrightyear{2023}

\acmYear{2023}
\acmVolume{XXXX}
\acmNumber{Y}
\acmDOI{XXXXXXX.XXXXXXX}
\acmISBN{}
\acmConference{Proceedings on Privacy Enhancing Technologies}
\usepackage[utf8]{inputenc}%(only for the pdftex engine)

\usepackage[T1]{fontenc}    % use 8-bit T1 fonts
\usepackage{url}            % simple URL typesetting
\usepackage{booktabs}       % professional-quality tables
\usepackage{microtype}      % microtypography

\usepackage{amsfonts}
\usepackage{multicol}
\usepackage{multirow}
\usepackage{mathtools}
\usepackage{algorithmic}
\usepackage{bmpsize}
\usepackage{xcolor}
\usepackage{lipsum}
\usepackage{textcomp}
\usepackage{bm}
\usepackage{gensymb}
\usepackage{nicefrac}
\usepackage{bbm}
\usepackage[ruled, vlined, linesnumbered]{algorithm2e}
\usepackage{tabularx}
\newcolumntype{b}{>{\centering\arraybackslash}X}
\newcolumntype{s}{>{\hsize=.5\hsize}X}
\usepackage{float}
\usepackage{pdfpages}
\usepackage{subcaption, booktabs}
\usepackage{graphicx}
\usepackage{tikz}
\usepackage{xspace}
\usetikzlibrary{automata, positioning, shapes.arrows, chains, math}
\usepackage[textsize=tiny]{todonotes}

\usepackage{booktabs,arydshln}
\makeatletter
\def\adl@drawiv#1#2#3{%
        \hskip.5\tabcolsep
        \xleaders#3{#2.5\@tempdimb #1{1}#2.5\@tempdimb}%
                #2\z@ plus1fil minus1fil\relax
        \hskip.5\tabcolsep}
\newcommand{\cdashlinelr}[1]{%
  \noalign{\vskip\aboverulesep
           \global\let\@dashdrawstore\adl@draw
           \global\let\adl@draw\adl@drawiv}
  \cdashline{#1}
  \noalign{\global\let\adl@draw\@dashdrawstore
           \vskip\belowrulesep}}
\makeatother
\usepackage{amsthm}
\theoremstyle{plain}
\newtheorem{thm}{Theorem}
\newtheorem{lem}{Lemma}
\newtheorem{prop}{Proposition}
\newtheorem{cor}{Corollary}
\theoremstyle{definition}
\newtheorem{df}{Definition}

\theoremstyle{remark}

\usepackage[capitalize, noabbrev]{cleveref}
\crefalias{df}{definition}
\crefalias{lem}{lemma}
\crefalias{prop}{proposition}
\crefalias{thm}{theorem}
\crefalias{cor}{corollary}
\crefalias{conj}{conjecture}
\crefalias{emlp}{example}
\crefalias{rem}{remark}
\Crefname{algocf}{Algorithm}{Algorithms}
\Crefformat{equation}{Equation~(#2#1#3)}
\crefformat{equation}{Eq.~(#2#1#3)}

\makeatletter
\DeclareRobustCommand\onedot{\futurelet\@let@token\@onedot}
\def\@onedot{\ifx\@let@token.\else.\null\fi\xspace}
\def\eg{\emph{e.g}\onedot} \def\Eg{\emph{E.g}\onedot}
\def\ie{\emph{i.e}\onedot} 
\def\cf{\emph{cf}\onedot}

\def\etal{\emph{et al}\onedot} \def\st{\emph{s.t}\onedot}
\usepackage{pifont}

\newcommand{\personalized}{individualized\@\xspace}
\newcommand{\Personalized}{Individualized\@\xspace}

\newcommand{\personalization}{individualization\@\xspace}

\SetKwComment{Comment}{/*}{*/}

\usepackage{ifthen}
\newboolean{review}  
\setboolean{review}{false}
\ifthenelse{\boolean{review}}{
\newcommand{\new}[1]{\textcolor{blue}{#1}}
\newcommand{\out}[1][]{\textcolor{blue}{[...]}}
}
{
\newcommand{\new}[1]{\textcolor{black}{#1}}
\newcommand{\out}[1][]{\textcolor{blue}{}}
}

\begin{document}

\title{Individualized PATE: Differentially Private Machine Learning with Individual Privacy Guarantees}
\titlenote{Accepted at \textit{23rd Privacy Enhancing Technologies Symposium}~(PETS~2023).}

\author{Franziska Boenisch}\authornote{The work was done while the author was at Fraunhofer AISEC.}\authornote{Both authors contributed equally.}
\email{franziska.boenisch@vectorinstitute.ai}
\affiliation{%
  \institution{Vector Institute}
    \city{Toronto}
  \country{Canada}
}

\author{Christopher Mühl}\authornotemark[3]
\email{christopher.muehl@fu-berlin.de}
\affiliation{%
  \institution{Free University Berlin}
      \city{Berlin}
  \country{Germany}
}
\author{Roy Rinberg} 
\email{roy.rinberg@columbia.edu}
\affiliation{%
  \institution{Columbia University}
        \city{New York}
  \country{USA}
}
\author{Jannis Ihrig}
\email{jannis.ihrig@fu-berlin.de}
\affiliation{%
  \institution{Free University Berlin}
        \city{Berlin}
  \country{Germany}
}
\author{Adam Dziedzic}
\email{adam.dziedzic@utoronto.ca}
\affiliation{%
  \institution{University of Toronto, Vector Institute}
      \city{Toronto}
  \country{Canada}
}

\renewcommand{\shortauthors}{Boenisch et al.}

\begin{abstract}
Applying machine learning (ML) to sensitive domains requires privacy protection of the underlying training data through formal privacy frameworks, such as differential privacy (DP). Yet, usually, the privacy of the training data comes at the cost of the resulting ML models' utility. One reason for this is that DP uses one uniform privacy budget $\varepsilon$ for all training data points, which has to align with the strictest privacy requirement encountered among all data holders. In practice, different data holders have different privacy requirements and data points of data holders with lower requirements can contribute more information to the training process of the ML models. To account for this need, we propose two novel methods based on the Private Aggregation of Teacher Ensembles (PATE) framework to support the training of ML models with individualized privacy guarantees. We formally describe the methods, provide a theoretical analysis of their privacy bounds, and experimentally evaluate their effect on the final model's utility using the MNIST, SVHN, and Adult income datasets. Our empirical results show that the individualized privacy methods yield ML models of higher accuracy than the non-individualized baseline. Thereby, we improve the privacy-utility trade-off in scenarios in which different data holders consent to contribute their sensitive data at different individual privacy levels.
\end{abstract}

%\keywords{privacy, machine learning, pate, individualized privacy}


\maketitle

\begin{table}
%\vspace{-2pt}
  \caption{\textbf{Number of generated labels by Standard vs our Individualized PATE} on the MNIST \new{and SVHN} datasets. \textit{D}: distribution of privacy groups (percentage wise), $\varepsilon$: privacy budget for a given group, \textbf{N}: number of generated labels, and \textbf{A}: accuracy of all votes generated by the respective method.  
 % The voting accuracies for all methods are $\approx 97\%$ on MNIST \new{and XXX on SVHN}.
  %Comparison between standard PATE and proposed individualized methods in terms of number of answered queries.
  The percentages of the three privacy groups are chosen according to~\cite{hdp} (first setup/row) and~\cite{upem} (second setup/row). 
  \new{We use the standard train-test split for MNIST and train the teacher models using the first 50K train samples while keeping the remaining 10K as the public dataset, and evaluating on the 10K standard test samples. 
  Similarly for SVHN, we split the train set into a public set of 10K samples and use the remaining 63257 train samples as the standard train set. Then, the test set is used for evaluation.
  We run the experiments three times and report the standard deviation.}
  }
  \label{tab:compare-pate}
  \begin{sc}
  \begin{center}
  %\vspace{-1.5pt}
  \small
  %\small
  \new{
  \begin{tabular}{ccccccc}
    \toprule
    & \multicolumn{3}{c}{Setup} & \textbf{PATE} & \textbf{Upsample} & \textbf{Weight} \\
    \hline
    \multirow{2}{*}{MNIST} & \textit{D} & \textit{34\%-43\%-23\%} & \textbf{N} & 365$\pm$2 & 1333$\pm$1 & 1312$\pm$8 \\
    & \textit{$\varepsilon$} & \textit{1.0-2.0-3.0} & \textbf{A} & 97.40 & 97.20 & 97.24 \\
    %\hline
    \cdashlinelr{1- 7}
    \multirow{2}{*}{MNIST} & \textit{D} & \textit{54\%-37\%-9\%} & \textbf{N} & 361$\pm$3 & 949$\pm$18 & 1894$\pm$10 \\
    & \textit{$\varepsilon$} & \textit{1.0-2.0-3.0} & \textbf{A} & 97.32 & 97.18 & 97.33 \\
    \hline
    \multirow{2}{*}{SVHN} & \textit{D} & \textit{34\%-43\%-23\%} & \textbf{N} & 90$\pm$1 & 394$\pm$5 & 409$\pm$3 \\
    & \textit{$\varepsilon$} & \textit{1.0-2.0-3.0} & \textbf{A} & 64.55& 64.33& 66.40 \\
    %\hline
    \cdashlinelr{1- 7}
    \multirow{2}{*}{SVHN} & \textit{D} & \textit{54\%-37\%-9\%} & \textbf{N} & 96$\pm$2 & 284$\pm$1 & 558 $\pm$1 \\
    & \textit{$\varepsilon$} & \textit{1.0-2.0-3.0} & \textbf{A}& 62.90& 62.85& 65.65\\
    \bottomrule
  \end{tabular}
  }
  \end{center}
  \end{sc}
\vspace{-4pt}
%\end{wraptable}
\end{table}

\section{Introduction}
%Machine Learning (ML) is being applied in an increasing number of sensitive domains, such as health care \citep{wiens2018machine, chen2017disease}, genetics and genomics \citep{libbrecht2015machine}, or hiring processes \citep{mahmoud2019performance}. Ensuring privacy of the data that the ML models are trained on plays an increasingly vital role. 
%However, research has shown that ML models are prone to privacy attacks \citep{shokri2017membership, fredrikson2015model, ganju2018property}.
%Such attacks allow malicious adversaries to learn what data points the ML model under attack has been trained on, what sensitive attributes the data exhibits, or how the data is distributed.

%One gold standard to protect data privacy is the mathematical framework of Differential Privacy (DP) \citep{dwork2006differential}.
%It allows learning potentially sensitive properties about a population of data as a whole while disclosing limited private information about individual data points.
%This is achieved through adding a controlled amount of statistical noise to the data or during the analysis to dissimulate sensitive properties.
%The amount of noise depends on a so-called privacy budget $\varepsilon$.
%Lower values of $\varepsilon$ produce higher amounts of noise that guarantee \emph{high privacy} but potentially \emph{reduce utility}~\citep{bagdasaryan2019differential, papernot2020tempered}.
%\todo{roei: I would remove or *significantly* compress the above 2 paragraphs}

Machine learning (ML) is increasingly applied in settings where training data is sensitive. At the same time, training data leakage is ubiquitous~\citep{shokri2017membership, fredrikson2015model, ganju2018property}, motivating approaches that integrate \textit{differential privacy} (DP)~\citep{dwork2006differential}. When properly applied, 
%DP can guarantee that information about individuals is not leaked through model predictions.
%such as health care, genomics, or job-candidate vetting~\citep{wiens2018machine, chen2017disease,libbrecht2015machine,mahmoud2019performance}. 
\out
%There are two prominent approaches to applying DP to ML models, namely the Differentially Private Stochastic Gradient Descent (DP-SGD) algorithm \citep{dp_sgd}, and the Private Aggregation of Teacher Ensembles (PATE) framework~\citep{pate_2017}.
% Both methods have trade-offs between the degree of privacy introduced by DP and the model's utility (measured as, for example, its accuracy).
\new{
DP guarantees that}
the amount of sensitive information that the trained ML models can potentially leak at inference time is bounded by the privacy budget $\varepsilon$.
\new{However, there exist}
trade-offs between the degree of privacy introduced by DP and the model's utility (measured as, for example, its accuracy).
\new{Furthermore}, we observe an important characteristic of current ML applications with DP: 
$\varepsilon$ is a single parameter that controls the protection level for the entire dataset, even if some data points in it are not sensitive at all. This coarse level of privacy parameterization seems extremely wasteful: intuitively, if large portions of the data need little protection whereas other parts are highly sensitive, then choosing an $\varepsilon$ tuned to sufficiently protect the sensitive data, and using it to protect \emph{all} data, might unnecessarily penalize the model's utility.

In addition to different data being inherently more or less sensitive, it is also known that in society, individuals have different attitudes towards privacy protection, and therefore, require their data to be protected at different levels%According to
~\cite{jensen2005privacy, berendt2005privacy}.
\out
%there exist at least three different groups of individuals, demanding high, average, and low privacy protection for their data, respectively.
%These groups are sometimes referred to as \emph{privacy fundamentalists}, \emph{privacy pragmatists} and \emph{privacy unconcerned}~\cite{taylor2003most}.
\new{Since current ML applications under DP only allow for setting a uniform privacy budget $\varepsilon$, even when the data holders have different privacy requirements,} the privacy budget would always have to be chosen according to 
%the privacy fundamentalists' requirements.
\new{the individuals with the highest requirements.}
However, given the privacy-utility trade-off mentioned above, it would be desirable not to always implement the highest privacy protections for all data points. 
Instead, allowing several individual privacy budgets according to the data holders' respective preferences can help to better leverage the training data, and increase the utility of the resulting ML model.

While approaches for supporting the specification of individual privacy preferences exist for statistical data analyses with DP~\cite{hdp, pdp}, to the best of our knowledge, no such frameworks exist in the context of ML.
Yet, there is a multitude of applications that already benefit from \personalized DP for data analysis, such as smart home~\cite{zhang2016personalized}, smart grid~\cite{bhattacharjee2021personalized}, and object localization~\cite{wang2018personalized, deldar2019pdp}---underlining the relevance of the topic and the need to extend individualized DP methods to ML.
To this end, in this work, we introduce \new{two} novel methods (\emph{upsampling} and \emph{weighting}) that extend the \new{Private Aggregation of Teacher Ensembles (PATE)~\citep{pate_2017}} algorithm\new{---one of the standard frameworks to implement DP in ML applications---}and support \personalized assignment of privacy budgets among the sensitive training data.
We first theoretically introduce \new{both} our methods and provide a detailed privacy analysis.
Then, we experimentally evaluate our methods' implications on the resulting model utility on the example of the MNIST \citep{mnist}, \new{SVHN}~\cite{netzer2011reading}, and Adult income~\citep{adult} datasets.
In particular, we study how different distributions of individual privacy preferences and respective privacy budgets influence the gained utility.
Our experiments highlight that in comparison to the standard PATE approach where the uniform privacy budget is determined by the data point with the highest privacy requirements, our \personalized PATE variants generate significantly more labels, and thereby increase utility of the student model.
The significant increase of generated labels by our upsampling and weighting method in comparison to standard PATE is visualized in \Cref{tab:compare-pate}.
The fraction of data points assigned to each of the three privacy groups is specified according to individuals' preferences observed within society by~\cite{jensen2005privacy, berendt2005privacy}.

%Depending on the privacy budget distribution and the PATE variant, we are able to produce approximately between $+60\%$ to $+950\%$ more labels for the MNIST, and $+60\%$ to $+890\%$ for the Adult income dataset. 
In summary, we make the following contributions:
\begin{itemize}
    \item Introduction of \new{two} novel \personalized PATE variants;
    \item Theoretical analysis of the respective privacy bounds;
    \item Experimental evaluation of utility improvements for the MNIST\new{, SVHN,} and Adult income dataset;
    \item Quantification of the effects of different privacy budget distributions on the gained utility; 
\end{itemize}

\vspace{-0.3cm}
\paragraph*{Ethical Implications}
In general, deciding on an adequate DP budget $\varepsilon$ in ML applications dealing with sensitive data is a challenging task.
This results from real-world implications of concrete values for $\varepsilon$ being poorly understood. 
Additionally, even the calculated privacy budgets $\varepsilon$ for the same application and data might decrease over time, when tighter bounds for their calculation are pushed forward \citep{moments_accountant}.
These inherent difficulties of choosing an adequate  $\varepsilon$ are also faced when assigning individual privacy budgets to data points.
In particular, one needs to make sure that no entity training an ML model with individual DP guarantees abuses their power and assigns poor levels of privacy to data that actually requires privacy protection.
We, therefore, suggest the use of our new \personalized PATE variants in settings that contain a process for obtaining informed consent of the data holders to process their data at a given  privacy level, such as~\citep{sorries2021privacy}.
This process should consist of (1) the identification of the individual privacy preferences~\citep{teltzrow2004impacts, kolter2009generating}, (2) the communication of the associated privacy risks and limitations (\eg~\citep{wachter2017counterfactual}), and (3) enabling meaningful decision-making processes by providing information about DP concerning sensitive data disclosure (\eg~\citep{xiong2020towards}).
\new{
In particular (2) must be implemented in a way that the risks are communicated clearly, such that nudging individuals into giving up their privacy will be prevented.
}
Moreover, we argue that, due to its difficult interpretability, individual data holders should not be in charge of choosing their numeric privacy budget $\varepsilon$, but, based on the information on potential risks and benefits decide on an abstract privacy level, such as "high", "average", or "low"~\citep{jensen2005privacy, berendt2005privacy, hoofnagle2014alan}.
Concrete numeric values $\varepsilon$ can then be fixed by the regulator or ethics committee in charge depending on the sensitive data itself and the application~\citep{bhat2020sociocultural}.
\new{
We argue that these values should be chosen such that even the lowest privacy budget still offers protection in practice~\cite{nasr2021adversary}.
This approach can be considered as a form of soft-paternalism~\cite{acquisti2009nudging} to protect privacy of the sensitive data held by individuals who are not concerned about the topic. 
%Furthermore, we will extend the ethical-implications section by a paragraph on prevention of nudging and a discussion on mechanisms, such as soft-paternalism, that ensure that the lowest possible privacy level for a user still offers meaningful protection.
%worth it to include a discussion on what protections could be added to ensure some minimum level of privacy protection, or to prevent nudging individuals to give up their privacy
}
\section{Notation \& Background}
\label{sec:Background}
% This section provides the theoretical foundation of the work at hand and introduces the notation used throughout the paper.
% \subsection{Notation}
\label{Notation}
We call $\mathcal{D}$ and $\mathcal{R}$ the sets of all possible data points, and all possible processing results that can be produced on them, respectively. Furthermore, two concrete datasets $D, D' \subseteq \mathcal{D}$ are called neighboring (written $D \sim D'$) if $D$ and $D'$ differentiate exactly in one data point. More specifically, they are called neighboring on $d$ (written $D \overset{d}{\sim} D'$) if they differ by any but exactly one data point $d \in \mathcal{D}$.

To refer to $\varepsilon$, we will use the term \emph{privacy budget} when expressing the privacy preference specified for a (group of) data point(s), and the term \emph{privacy costs} when referring to the proportion of budget being already consumed in a DP-based mechanism.

All $log$ values in this work are based on the natural logarithm.
Furthermore, $\mathbb{P} [\cdot]$ denotes the probability of an event according to an adequate probability measure, and $\mathbb{E} [\cdot]$ outputs the expected value of a given random variable.

\subsection{Differential Privacy}
\label{sec:Differential Privacy}
DP formalizes the idea of limiting the influence of individual data points on the results of analyses conducted on a whole dataset.
One relaxation of the standard definition of DP is called $(\varepsilon, \delta)$-DP.
\begin{df}[\cf \cite{dp}, Def. 2.4]%$(\varepsilon, \delta)$-Differential Privacy
\label{df:Differential Privacy}
Let  $D, D' \subseteq \mathcal{D}$ be two neighboring datasets.
Let $M \colon \mathcal{D}^* \rightarrow \mathcal{R}$ be a mechanism that processes arbitrarily many data points.
$M$ satisfies $(\varepsilon, \delta)$-DP with $\varepsilon \in \mathbb{R}_+$ and $\delta \in [0, 1]$ if for all datasets $D \sim D'$, and for all result events $R \subseteq \mathcal{R}$
\begin{align}
    \mathbb{P}\left[M(D) \in R\right] \leq e^\varepsilon \cdot \mathbb{P}\left[M(D') \in R\right] + \delta \, .
\end{align}
\end{df}
Thereby, it expresses the guarantee that a single data point cannot alter the probability of any processing result by a factor larger than $\exp(\varepsilon)$.
The second parameter $\delta$  specifies a small density of probability on which the upper bound does not have to hold.

In ML, data is usually processed multiple times to train a model, \eg by conducting several training epochs. 
This process can be considered as a \emph{composition} of mechanisms that each have privacy costs. 
The following composition theorem states how DP behaves under composition as follows.
\begin{prop}[\cf \citep{ed-dp}, Thm. 3.16]%Composition of DP
\label{prop:Composition of DP}
Let $\mathcal{R}_1, \mathcal{R}_2$ be two arbitrary result spaces. Let further $M_1 \colon \mathcal{D}^* \rightarrow \mathcal{R}_1$, $M_2 \colon \mathcal{D}^* \rightarrow \mathcal{R}_2$ be mechanisms that satisfy $(\varepsilon_1, \delta_1)$- and $(\varepsilon_2, \delta_2)$-DP, respectively. Then, the composition $M_3(D) \mapsto (M_1(D), M_2(D))$ satisfies $(\varepsilon_1 + \varepsilon_2, \delta_1 + \delta_2)$-DP.
\end{prop}
The proof can be found in Appendix B of \citep{ed-dp}.

\subsection{Rényi Differential Privacy}
\label{sec:Rényi Differential Privacy}
\Cref{prop:Composition of DP} shows that under composition, $(\varepsilon, \delta)$-DP quickly leads to a combinatorial explosion of parameters. A smoother composition of privacy bounds can be achieved by using Rényi Differential Privacy (RDP) \citep{rdp} which is based on the Rényi divergence (see Definition~\ref{df:Rényi Divergence} in \Cref{Appendix_Background}).

\begin{df}[\cf \citep{rdp}, Def. 4]%Rényi Differential Privacy
\label{df:Rényi Differential Privacy}
A mechanism $M \colon \mathcal{D}^* \rightarrow \mathcal{R}$ satisfies $(\alpha, \varepsilon)$-RDP with $\alpha \in \mathbb{R}_+ \setminus \{1\}$ and $\varepsilon \in \mathbb{R}_+$ if for all datasets $D \sim D'$ and for all result events $R \subseteq \mathcal{R}$
\begin{align}
    \mathbb{D}_\alpha\left[f_{M(D)} \parallel f_{M(D')}\right] \leq \varepsilon \, .
\end{align}
Here, $f_{M(D)}$ and $f_{M(D')}$ are the probability distributions of the results of $M$ on $D$ and $D'$, respectively.
\end{df}
In Lemma~\ref{lem:Composition of RDP}, and Lemma~\ref{lem:RDP to DP} in \Cref{Appendix_Background}, we show the composition and transformation from RDP to DP guarantees, respectively.

\subsection{\Personalized Differential Privacy}
\label{sec:Personalized Differential Privacy}
\Personalized DP, similar to \citep{pdp, hdp, pdp_other, per_instance-dp}, allows accounting for privacy for data points individually.
\begin{df}[\cf \citep{pdp}, Def. 6]%\Personalized Differential Privacy
\label{df:Personalized Differential Privacy}
For any data point $d \in \mathcal{D}$, $M$ satisfies $(\varepsilon_d, \delta_d)$-DP with $\varepsilon_d \in \mathbb{R}_+$ and $\delta_d \in [0, 1]$ if for all datasets $D \overset{d}{\sim} D'$, and for all result events $R \subseteq \mathcal{R}$
\begin{align}
    \mathbb{P}\left[M(D) \in R\right] \leq e^{\varepsilon_d} \cdot \mathbb{P}\left[M(D') \in R\right] + \delta_d \, .
\end{align}
\end{df}
Accounting privacy per data point can also be applied to different DP variants, such as RDP.
Properties like composition and transformation apply to RDP analogously to the original concepts.

\subsection{PATE}
\label{sec:PATE}
The PATE framework \citep{pate_2017} can be used to perform supervised ML with DP guarantees. 
Therefore, the set of private labeled training data is split among a pre-defined number of so-called \emph{teacher} models and each teacher is trained on their partition of the data.
Afterward, the knowledge gained by the teachers from the private training data is transferred to a public so-called \emph{student} model.
To do so, the teachers label a public and unlabeled dataset as training data for the student.
Privacy protection for the teachers' sensitive training data is obtained by adding DP noise during the labeling process, and by the fact that the student does not get to interact with the sensitive data, but instead uses the public dataset for training.
See \Cref{fig:PATE Schemes}a in the Appendix for an overview of the approach.

The DP noise addition in the labeling process determines the privacy level of PATE.
To obtain a label for a public data point, each teacher issues a vote for a specific class.
These votes are aggregated with the Gaussian NoisyMax Aggregator as follows: 

\begin{df}[\cf \citep{pate_2017}, Sec. 2.1]%Vote Count
\label{df:Vote Count}
Let $\mathcal{X}, \mathcal{Y}$ be the feature space, and the set of classes corresponding to data distribution, respectively.
Further, let $t_i \colon \mathcal{X} \rightarrow \mathcal{Y}$ be the $i$-th teacher of a teacher ensemble of size $k \in \mathbb{N}$. The vote count $n \colon \mathcal{Y} \times \mathcal{X} \rightarrow \mathbb{N}$ of any class $j \in \mathcal{Y}$ for any data point $x \in \mathcal{X}$ is:
\begin{align}
    n_j(x) \coloneqq \sum\limits_{i=1}^{k} \mathbbm{1} \left( t_i (x) = j \right) \; .
\end{align}
The characteristic function $\mathbbm{1} \colon \{\bot, \top\} \rightarrow \{0, 1\}$ maps 'true' to $1$ and 'false' to $0$. Note that the vote count depends on the teachers and, therefore, also on their training data.
\end{df}
\begin{df}[\cf \citep{pate_2018}, Sec. 4.1]%Gaussian NoisyMax Aggregator
Let $n_j$ be the vote count as defined in \Cref{df:Vote Count} for each class $j \in \mathcal{Y}$. Then, the Gaussian NoisyMax (GNMax) aggregation method with parameter $\sigma \in \mathbb{R_+}$ on any data point $x \in \mathcal{X}$ is given by:
\begin{align}
\label{eq:gnmax}
    \mathrm{GNMax}_\sigma(x) \coloneqq \arg \underset{j \in \mathcal{Y}}{\max} \left\{n_j(x) + \mathcal{N}\left(0, \sigma^2\right)\right\} \,.
\end{align}
\end{df}
The Gaussian noise is sampled from a normal distribution $\mathcal{N}(\mu, \sigma^2)$ with mean $\mu = 0$ and variance $\sigma^2$. 

As an extension of the original GNMax Aggregator, Papernot \etal \citep{pate_2018} proposed the \emph{Confident-GNMax Aggregator}:
\begin{align}
\label{eq:confident-gnmax}
\underset{j \in \mathcal{Y}}{\max}\{ n_j(x)\} +  \mathcal{N}(0, \sigma_{T}^2) > T
\end{align} 
that only labels data points for which the consensus of the teachers exceeds a pre-defined threshold $T$. See \citep{pate_2018} for a formalization of this idea.

\section{Related Work}
\label{sec:Related Work}
%\todo{Should Related Work be placed before Background?}\adam{No, this is fine.}
We present related work on \personalized privacy, DP, and privacy-preserving ML techniques.

\new{
\subsection{Individualizing Privacy}
According to~\cite{jensen2005privacy, berendt2005privacy}, there exist at least three different groups of individuals, demanding high, average, and low privacy protection for their data, respectively.
These groups are sometimes referred to as \emph{privacy fundamentalists}, \emph{privacy pragmatists} and \emph{privacy unconcerned}~\cite{taylor2003most}.
In general and ML-based applications with DP that deal with data of individuals from all three groups, the privacy budget would always have to be chosen according to the privacy fundamentalists' requirements.
This can lead to unfavorable privacy-utility trade-offs in the respective application.
Hence, it would be desirable to use \personalized privacy budgets to improve model utility while complying with each individual's personal privacy requirements.
}

\subsection{Techniques for \Personalized DP}
\label{Techniques for Personalized DP}
Several techniques for implementing \personalized privacy guarantees with DP for data analysis outside of the scope of ML have been proposed.

One of the first techniques for \personalized DP was proposed by Alaggan~\etal~\citep{hdp}.
Their \emph{stretching mechanism} scales data points individually before perturbing them with statistical noise.
As a consequence, the DP noise affects each point with individual intensity.
Jorgensen~\etal~proposed two additional methods~\citep{pdp}.
Their first \emph{sample mechanism} excludes particular data points from being included in the respective data analysis with a probability according to their privacy preferences.
Their second \emph{personalized exponential mechanism} assigns probabilities to processing results according to individual data points' privacy requirements.
These probabilities are then used to randomly select the final processing result for a dataset.
Two \emph{partitioning algorithms} were introduced by Li~\etal~\citep{partitioning}.
These separate process groups of sensitive data, each with an individual privacy preference.
In a similar vein, Niu~\etal~\cite{niu2021adapdp} described a utility-aware sub-sampling mechanism to implement \personalized DP guarantees.
Ebadi~\etal~\cite{pdp_other} put forward a \emph{personalized DP} mechanism that relies on excluding data points from the analysis once their respective privacy budgets are exceeded. 
Their algorithm is designed for live databases in mind where individual data points might not only require individual privacy protection but can also be added to the data analysis at different points in time.
As a consequence, each data point also needs individual privacy budget accounting.

Since all the proposed \personalized DP mechanisms are designed for privacy-preserving data analysis on datasets and databases, rather than on ML models, they are not directly applicable to our setting.
Note, however, that our\out~weighting mechanism is inspired by\out~the stretching mechanism.

\subsection{DP Mechanisms for ML}
\label{DP Mechanisms for ML}

%There are two prominent approaches to applying DP to ML models, namely the Differentially Private Stochastic Gradient Descent (DP-SGD) algorithm \citep{dp_sgd}, and the Private Aggregation of Teacher Ensembles (PATE) framework~\citep{pate_2017}.
%In both, the amount of sensitive information that the trained ML models can potentially leak at inference time is bounded by the privacy budget $\varepsilon$. Both methods have trade-offs between the degree of privacy introduced by DP and the model's utility (measured as, for example, its accuracy).
\new{}

PATE is not the only approach that can be used to apply DP in ML workflows.
Another commonly used approach is the \emph{Differentially Private Stochastic Gradient Descent (DP-SGD)}~\citep{dp_sgd}.
In DP-SGD, privacy is achieved by first limiting the changes to an ML model that each individual data point can cause.
This is done by clipping model gradients on a per-example basis during training. 
Then, to achieve DP guarantees, noise is added to the gradients before the model parameters are updated with them.
Privacy costs of DP-SGD are accounted for through the \emph{moments accountant}~\citep{moments_accountant}.
In this approach, multiple moments of the privacy loss random variable are calculated to obtain a DP bound by using the standard Markov inequality.

In the scope of DP-SGD, Feldman and Zrnic~\citep{renyi_filter} proposed an individual per-data point privacy accounting using \emph{RDP filters}.
Similarly, Jordon~\etal~\citep{pdp_accountant} personalized the moments' accountant by dividing it into an \emph{upwards} and a \emph{downwards moments accountant} which are composed to a \emph{personalized moments accountant} to provide data-dependent DP bounds individually per data point.
\new{In a similar vein, Yu~\etal~\cite{indi_dpsgd_accounting} proposed individualized privacy accounting for DP-SGD based on the gradient norms of the individual data points.}
While both our and these \new{three} works aim at improving the privacy-utility trade-offs in ML with DP, their work differs from ours in the problem setting.
Our work sets out to address the problem of supporting data holders in \emph{specifying and implementing} their individual privacy preferences, whereas their work aims at \emph{accounting} for per-data point loss incurred during training of the ML model. 
Therefore, they assign a uniform privacy budget over the whole training dataset and then provide a tighter per-data point analysis of privacy loss.
Based on this tighter analysis, data points can be excluded from training once their individual privacy budget is exhausted, while other data points can still be used for further training.
So, rather than asking the question that our work is concerned with, namely \emph{What impact does assign individual privacy budgets to the training data have on the resulting ML model utility?}, they address the question \emph{What privacy loss is incurred to each individual data point by the given algorithm on the given dataset?}
As a consequence, while utility gain in their method is solely due to leveraging each data point based on its individual privacy loss, our method can offer an additional utility gain due to supporting individual privacy budgets per data point.

Note that, due to the different structures of the approaches, the \personalized privacy accounting of DP-SGD from \citep{renyi_filter} or \cite{pdp_accountant} cannot be directly applied to PATE.
Therefore, our methods extend PATE's inherent privacy accounting to \personalized accounting.

% Supporting \personalized privacy  in PATE is of high relevance though, since PATE supports a wide range of scenarios in which DP-SGD is not that easily applicable:
% While DP-SGD is mainly applicable in centralized learning applications with a non-convex model architecture, PATE can also be applied to distributed learning scenarios even when different participants hold different types of models with different training algorithms and architectures.
% Additionally, since in DP-SGD, privacy bounds depend on the model parameters, for large models, the privacy guarantees will degrade which is not the case for PATE \citep{pate_2017}. 

\section{\Personalized Extensions for PATE}
\label{sec:Personalized Extensions for PATE}

% \begin{figure*}[ht]
% \includegraphics[width=0.1\textwidth, trim=2.4cm 1.2cm 3.6cm 2cm]
% {Images/PATE_scheme}
% \caption{\textbf{Overview on modifications to PATE incurred by our \personalized variants}. While \emph{upsampling} modifies the sampling and assignment of the sensitive training data to the teacher models, \emph{vanishing} and \emph{weighting} alter the aggregation of the teacher votes in the labeling process.}
% \label{fig:pate_modifications}
% \end{figure*}

To implement individual privacy requirements of sensitive training data points, we propose \new{two} novel \personalized variants of PATE, namely \emph{upsampling} \out and \emph{weighting}.
Each variant modifies the original PATE algorithm in some aspects to provide \personalized privacy.

Each of our individualized variants overcomes the limitation of non-\personalized PATE where the uniform privacy budget $\varepsilon$ has to be chosen according to the highest privacy requirement encountered in the sensitive training data.
Thereby, our variants allow us to generate more labels than PATE, and to train a student model with higher utility. In the case when all sensitive data points require the same privacy, our \personalized PATE variants are equivalent to non-\personalized standard PATE.

In this section, we first introduce the ideas behind our\out~variants and then perform an evaluation of their privacy levels.
Therefore, we rely on the privacy analysis of the original PATE algorithm~\citep{pate_2017}, and extend it to our individual variants by analyzing the sensitivity of the vote counts.
For \new{both} our variants, we also propose concrete algorithms illustrating how they can be implemented. Note, however, that these algorithms only represent possible instantiations of the implementations.
In general, what the algorithms should ensure is that \new{our variants} yield setups in which data points with different privacy budgets exceed their respective budget at approximately the same number of generated labels.
This is because label generation in \personalized PATE stops once any data point exceeds their privacy budget.
%\todo{This is actually a point of improvement, right? We could potentially use the data points that still have budget left for longer... So maybe we don't want to state that here so explicitly?!}
In practice, when training with \personalized PATE, model owners can simply observe the privacy budget consumption in the labeling process.
By identifying the best parameters
%\todo{I also feel not comfortable stating that - because it sounds like a somehow computationally expensive hyperparameter search} 
for each variant, such that the points' budget is exceeded at approximately the same number of generated labels, the model owner can then make sure that all privacy budgets are fully leveraged, and the highest number of labels is generated.
This in turn, leads to the best student model utility.

\subsection{Upsampling Mechanism}
\label{sec:Upsampling}
Our \emph{upsampling} mechanism relies on duplicating sensitive data such that overlapping data-subsets can be allocated to different teachers.
Thereby, data with higher privacy budgets is learned by a higher number of teachers.
The upsampling mechanism stands in contrast to the original PATE algorithm where \emph{disjoint} data partitions are passed to the teachers. 
Since data duplicates extend the amount of training data, they allow for two possible modifications of PATE: (1)~keeping the number of teachers constant and allocating more training data to each teacher, or (2)~keeping the number of training data points per teacher constant and increasing the number of teachers.
Our experimental evaluation indicates that (2)~yields a higher utility gain of upsampling PATE.
Intuitively, the teachers perform already reasonably well with the initial amount of training data, and allocating more data to them yields only marginal performance gains.
In contrast, having more teachers participate in the voting results in more accurate vote counts with less variance due to statistical randomness.
As a consequence, we implement upsampling according to (2)~with a constant number of training data points per teacher as specified in \Cref{alg:upsampling}.
The algorithm ensures that points are duplicated by an integer according to the privacy budget ratios, since only entire data points (and not fractions of a point) can be assigned to a teacher model. 
\new{See \Cref{fig:PATE Schemes}b in the Appendix for a visualization of the approach.}
%The algorithm presents one of the methods of how to optimally upsample the data points according to their privacy budgets.\todo{This sentence, for me, comes out of nowhere. I would not write it here. Instead, either at the beginning of the section for all methods, or at the end of the section. See my suggestion written in red.}

%%%%%% Adams previous per-group version:
% \begin{algorithm}
% \caption{Prepare training data for teacher models in the \textbf{upsampling} method.}\label{alg:upsampling}

% \SetKwInput{KwData}{Input}
% \KwData{Privacy budget $\varepsilon_j$ for each privacy group $g_j$, $j\in{1,...,G}$, precision $p \in \mathbb{N}$.}
% \KwResult{Upsampling factor $u_j$ for each privacy group $g_j$.}
% \SetAlgoLined
% \For{Each privacy group $g_j$}{
%     $\bar{\varepsilon}_j \gets \varepsilon_j \cdot 10^{p}$\Comment*[r]{Upscale budgets}
% }
% $d \gets $Greatest Common Divisor(${\bar{\varepsilon}_1,\dots,\bar{\varepsilon}_G}$)\;
% \For{Each privacy group $g_j$}{
%     $u_j \gets \frac{\bar{\varepsilon}_j}{d}$\;
% }
% \end{algorithm}

\begin{algorithm}
\caption{Prepare training data for teacher models in the \textbf{upsampling} method.}\label{alg:upsampling}

\SetKwInput{KwData}{Input}
\KwData{Privacy budgets $\{\varepsilon_d\}$ for each data point $d$%privacy group $g_j$, $j\in{1,...,G}$
, precision $p \in \mathbb{N}$.}
\KwResult{Upsampling factor $u_d$ for each data point $d$.}
\SetAlgoLined
$\{\varepsilon_1,\dots,\varepsilon_j\} \gets unique(\{\varepsilon_d\})$\Comment*[r]{Get unique budgets} %values among all points}
\For{Each $\varepsilon_j$}{
    $\bar{\varepsilon}_j \gets \varepsilon_j \cdot 10^{p}$\Comment*[r]{Upscale budgets}
}
$D \gets $Greatest Common Divisor(${\bar{\varepsilon}_1,\dots,\bar{\varepsilon}_G}$)\;
\For{Each $\bar{\varepsilon}_j$}{
    $u_d \gets \frac{\bar{\varepsilon}_j}{D}$\;
}
\end{algorithm}

We call the PATE aggregator for our upsampling approach \emph{upsampling GNMax (uGNMax)}.
It applies the \emph{upsampling vote count} which is defined as follows:
\begin{df}[Upsampling Vote Count]
\label{df:Upsampling Vote Count}
Let $t_i \colon \mathcal{X} \rightarrow \mathcal{Y}$ be the $i$-th out of $k \in \mathbb{N}$ teachers.
Let further $N \in \mathbb{N}$ be the number of sensitive data points and $m_i \in \{0, 1\}^N$ a mapping that describes which points are learned by $t_i$.
The upsampling vote count $\ddot{n} \colon \mathcal{Y} \times \mathcal{X} \rightarrow \mathbb{N}$ of any class $j \in \mathcal{Y}$ for any data point $x \in \mathcal{X}$ is
\begin{align}
    \ddot{n}_j(x) \coloneqq \sum\limits_{i=1}^{k} \mathbbm{1} \left( t_i (x) = j \right) \; .
\end{align}
\end{df}

Although the definition for the upsampling vote count looks the same as the non-\personalized vote count (\Cref{df:Vote Count}), their sensitivities differ due to data points to be learned by several teachers (see \Cref{Upsampling Sensitivity} in \Cref{sub:pate_privacy_evaluation_private}). 
\out

\subsection{Weighting Mechanism}
\label{sec:Weighting}
Our \emph{weighting} mechanism \out modifies the aggregation of teacher votes.
It does so by weighting individual teachers' votes higher or lower depending on their training data points' privacy requirements.
Therefore,\out~sensitive data points \new{that have the same privacy budget~$\varepsilon_j$, which we call a privacy group~$g_j$}, have to be allocated to the same teacher(s). 
In~\Cref{alg:weighting}, we present how weights $w_i$ can be assigned to the teachers.
\new{A visualization of the weighting mechanism is provided in \Cref{fig:PATE Schemes}d in the Appendix.}

\begin{algorithm}
\caption{Assign weights to teacher models in the \textbf{weighting} method.}\label{alg:weighting}
\SetKwInput{KwData}{Input}
\KwData{Privacy budget $\varepsilon_j$ and number of teachers $n_j$ for each privacy group $g_j$, $j\in{1,...,G}$, and total number of teachers $k$.}
\KwResult{Weight $w_i$ for each teacher $t_i$.}
$\mathcal{E} \gets \sum_{j=1}^{G} \varepsilon_j$\;
\For{Each privacy group $g_j$}{
    $\bar{\varepsilon}_j \gets \frac{\varepsilon_j}{\mathcal{E}}$\Comment*[r]{Relative privacy budget}
    $\bar{n}_j \gets \frac{n_j}{k}$\Comment*[r]{Relative group size}
    $\bar{w}_j \gets \bar{\varepsilon}_j \cdot \bar{n}_j$\; 
}
$\mathcal{W} \gets \sum_{j=1}^{G} \bar{w}_j$\;
\For{Each privacy group $g_j$}{
    $w_j \gets \frac{\bar{w}_j}{\mathcal{W}} \cdot k$\Comment*[r]{Make sum of weights match $k$}
    \For{Each teacher $t_i$ with data from $g_j$}{
        $w_i \gets w_j$\;%\Comment*[r]{Assign teacher weight according to privacy group of their training data}
    }
}
%Carry out voting using assigned weights.
% $i\gets 10$\;
% \eIf{$i\geq 5$}
% {
%     $i\gets i-1$\;
% }{
%     \If{$i\leq 3$}
%     {
%         $i\gets i+2$\;
%     }
% }
\end{algorithm}

%%% Adam's version per group
% \begin{algorithm}
% \caption{Assign weights to teacher models in the \textbf{weighting} method.}\label{alg:weighting}
% \SetKwInput{KwData}{Input}
% \KwData{Privacy budget $\varepsilon_j$ and number of teachers $n_j$ for each privacy group $g_j$, $j\in{1,...,G}$, and total number of teachers $k$.}
% \KwResult{Weight $w_j$ for each privacy group $g_j$.}
% $\mathcal{E} \gets \sum_{j=1}^{G} \varepsilon_j$\;
% \For{Each privacy group $g_j$}{
%     $\bar{\varepsilon}_j \gets \frac{\varepsilon_j}{\mathcal{E}}$\Comment*[r]{Relative privacy budget}
%     $\bar{n}_j \gets \frac{n_j}{k}$\Comment*[r]{Relative group size}
%     $\bar{w}_j \gets \bar{\varepsilon}_j \cdot \bar{n}_j$\; 
% }
% $\mathcal{W} \gets \sum_{j=1}^{G} \bar{w}_j$\;
% \For{Each privacy group $g_j$}{
%     $w_j \gets \frac{\bar{w}_j}{\mathcal{W}} \cdot k$\Comment*[r]{Make sum of weights match $k$}    
% }
% %Carry out voting using assigned weights.
% % $i\gets 10$\;
% % \eIf{$i\geq 5$}
% % {
% %     $i\gets i-1$\;
% % }{
% %     \If{$i\leq 3$}
% %     {
% %         $i\gets i+2$\;
% %     }
% % }
% \end{algorithm}

We call the aggregation method of this PATE variant \emph{weighting GNMax (wGNMax)}. 
Its vote count mechanism is defined as follows:

\begin{df}[Weighting Vote Count]
\label{df:Weighting Vote Count}
Let $t_i \colon \mathcal{X} \rightarrow \mathcal{Y}$ be the $i$-th out of $k \in \mathbb{N}$ teachers.
Let further $N \in \mathbb{N}$ be the number of sensitive data points and $m_i \in \{0, 1\}^N$ a mapping that describes which points are learned by $t_i$.
Moreover, let $w_i \in \mathbb{R}_+$ be the weight of $t_i$ for all $i \in \{1, \ldots, k\}$.
The weighting vote count $\tilde{n} \colon \mathcal{Y} \times \mathcal{X} \rightarrow \mathbb{N}$ of any class $j \in \mathcal{Y}$ for any unlabeled public data point $x \in \mathcal{X}$ is
\begin{align}
    \tilde{n}_j \left( x \right) \coloneqq \sum\limits_{i=1}^k w_i \cdot \mathbbm{1} \left(t_i (x) = j \right) \;.
\end{align}
\end{df}

\new{
As a particular variant of the weighting-mechanism, we also evaluate cases where some teachers have a zero-weight during some votings.
We call this variant the \textbf{Vanishing Mechanism}.
Intuitively, individualized privacy guarantees in the vanishing-  result from teachers contributing their information to more or less voting processes, depending on their data points' lower or higher privacy requirements, respectively.
See \Cref{sec:Vanishing} for details on the vanishing mechanism and its privacy assessment.
However, our experimental evaluation highlights that this approach, in general, yields low utility.
We suspect that this is due to the resulting reduced size of the teacher ensemble.}

\subsection{Privacy Evaluation}
\label{Privacy Evaluation of Personalized PATE}

\begin{table*}[t]
\centering
{
\new{
\begin{tabularx}{\textwidth}{|b|bbbbbb|}%{|l|cccccc|}
\hline
\textbf{Variant} & \textbf{Manipulation} & \textbf{Distributed} & \textbf{Privacy-budget} & \textbf{Sensitivity} & \textbf{Parameter changes} & \textbf{RDP privacy bound} \\
\hline
Upsampling & dataset & no & per data-point $d$ &  $u_d$ (how often $d$ is upsampled) & $k$, $\sigma$, $\sigma_T$, $T$ scaled according to $u_d$& $(\alpha, (u_d)^2 \cdot \nicefrac{\alpha}{\sigma^2})$ \\
Weighting & teacher aggregation & yes & per teacher $i$& $w_i$ (weight of teacher $i$)& N/A & $(\alpha, (w_i)^2 \cdot \nicefrac{\alpha}{\sigma^2})$\\
\hline
\end{tabularx}
}}
\caption{\new{\textbf{Summary of our individualized PATE variants}.
The table shows the properties of and privacy guarantees achieved by our mechanisms. \textbf{Manipulation}: what part of standard PATE is adapted; \textbf{Distributed}: mechanism suitable when data is distributed over different parties; \textbf{Privacy-budget}: how fine-grained can individual privacy budget be assigned; \textbf{Sensitivity}: sensitivity for teacher voting; \textbf{Parameter changes}: what parameters of standard PATE need to be adapted; \textbf{RDP privacy bound}: loose bound for privacy calculation. Calculation of both variants' tight bound is shown in \Cref{cor:Scaling Invariance of the Individual Loose Bound}.
}}
\label{tab:mechanisms}
\end{table*}

The privacy calculation of \personalized PATE differs from the standard (non-\personalized) PATE in that it is done for particular data points or groups of data points separately, rather than for the whole dataset. 
We first introduce the general elements of the privacy analysis for the standard PATE which is shared by our \new{two novel variants}.
Then, we evaluate the \personalized privacy guarantees of each variant depending on its vote count and aggregation mechanism.  
\new{\Cref{tab:mechanisms} summarizes our two methods, their differences, and their respective privacy guarantees.}

\subsubsection{Privacy Evaluation of Standard PATE}
\label{sec:Privacy Evaluation of Non-Personalized PATE}

A key element of privacy calculation in PATE is the aggregation mechanism.
PATE's GNMax Aggregator is a function of a Gaussian mechanism. 
\begin{df}[\cf \citep{ed-dp}, Sec. 3.5.3]%Gaussian Mechanism
\label{df:Gaussian Mechanism}
Let $f \colon \mathcal{D}^* \rightarrow \mathbb{R}^z$ with $z \in \mathbb{N}$ be any real-valued function and let $\sigma \in \mathbb{R}_+$ be any positive real. Then, the Gaussian mechanism of $f$ with standard deviation $\sigma$ is
\begin{align}
    M_{f, \sigma}(x) \coloneqq f(x) + \mathcal{N}(0, \sigma^2) \, .
\end{align}
\emph{Note:} the same random noise is added to $f(x)$ in each dimension.
\end{df}

Gaussian mechanisms have RDP costs depending on $\sigma$.
\begin{lem}[\cf \citep{rdp}, Prop. 7]%RDP Guarantee of the Gaussian Mechanism
\label{lem:RDP Guarantee of the Gaussian Mechanism}
Let $\sigma \in \mathbb{R}_+$ and let $f \colon \mathcal{D}^* \rightarrow \mathbb{R}$ be a real-valued function with sensitivity $\Delta_f \coloneqq \underset{\mathcal{D} \sim \mathcal{D}'}{\max} \left\Vert f(D) - f(D')\right\Vert_2$. Then, the Gaussian mechanism $M_{f, \sigma}$ satisfies $(\alpha, \Delta_f^2 \cdot \nicefrac{\alpha}{2 \sigma^2})$-RDP for all $\alpha \in \mathbb{R}_+ \setminus \{1\}$.
\end{lem}
\Cref{lem:RDP Guarantee of the Gaussian Mechanism} is proven in \citep{rdp}.
The resulting RDP costs can be transformed into $(\varepsilon, \delta)$-DP costs using \Cref{lem:RDP to DP}. 

%A rough estimate of the privacy costs that arise in PATE can be given by the data-independent \emph{loose~bound}.
The data-independent \emph{loose~bound} privacy costs that arise in PATE are given by:
\begin{lem}[\cf \citep{pate_2018}, Prop. 8]%Loose Bound of the GNMax
\label{lem:Loose Bound of the GNMax}
The GNMax aggregator satisfies $(\alpha, \nicefrac{\alpha}{\sigma^2})$-RDP for all $\alpha \in \mathbb{R}_+ \setminus \{1\}$.
\end{lem}
The intuition behind it is that in PATE, each data point of the training dataset is learned by exactly one teacher and is potentially able to change this teacher's vote.
Since DP guarantees are expressed for neighboring datasets that differ in exactly one data point $d$, in the worst case, $d$ changes the vote count for two classes (reduce one class count by one, and increase another class count by one).
Thus, a teacher voting can be considered as the composition of two Gaussian mechanisms each with sensitivity $\Delta_f = 1$ and parameter $\sigma$ equal to the standard deviation of the Gaussian noise. 
Putting the standard deviation of one into \Cref{lem:RDP Guarantee of the Gaussian Mechanism}, and applying composition of two Gaussian mechanisms, this yields the term specified in the loose bound.

In addition, it is also possible to obtain a tighter data-dependent bound for privacy estimation in PATE as defined in \citep{pate_2018}.
See Lemma~\ref{lem:Tight Bound of the GNMax} in \Cref{Appendix_Background} for a definition of this \emph{tight bound}.

\subsubsection{Privacy Evaluation of \Personalized PATE}
\label{sub:pate_privacy_evaluation_private}
All our new aggregation mechanisms apply \personalized vote counts $\bar{n} \colon \mathcal{Y} \times \mathcal{X} \rightarrow \mathbb{N}$ whose sensitivities are no longer $\Delta_f = 1$, but are determined individually for particular data points (or groups of data points).
Therefore, in the privacy analysis, we need to calculate their privacy bounds based on the mechanisms' individual sensitivities and the general privacy bounds of PATE.

The individual sensitivity of any function $f \colon \mathcal{D}^* \rightarrow \mathbb{R}^z$ with $z \in \mathbb{N}$ regarding any data point $d \in \mathcal{D}$ can be defined as $\Delta_{f, d} \coloneqq \underset{D \overset{d}{\sim} D'}{\max} \left\Vert f(D) - f(D')\right\Vert_2$.
The following propositions formalize the individual sensitivity of the vote counts in our \personalized PATE mechanisms.

\begin{prop}[Upsampling Sensitivity]
\label{Upsampling Sensitivity}
Let $d \in \mathcal{D}$ be any sensitive data point. Let $u_d \in \mathbb{N}$ be the number of duplicates of $d$ (incl. the original $d$). Then, the individual sensitivity of the vote count, regarding $d$, in upsampling PATE is
\begin{align}
    \Delta_{\mathrm{upsampling}, d} = u_d \text{.}
\end{align}
\end{prop}
\begin{proof}
In upsampling PATE, every teacher that is trained on data point $d \in \mathcal{D}$ can have a different vote for neighboring datasets that differ in $d$.
For each duplicate of $d$, this results in an increase of one vote count and a decrease of another one.
Let $t_{(d)}$ be the set of teachers trained on $d$.
Assume that all $u_d$ votes of $t_{(d)}$ would have changed if $d$ were different.
From the perspective of $d$, the voting can then be considered as a composition of $2 \cdot |\mathcal{Y}|$ Gaussian mechanisms (some might have a sensitivity of zero \st they have no privacy costs).
For each class $j \in \mathcal{Y}$ there are two Gaussian mechanisms, one with sensitivity equal to the number of votes of $t_{(d)}$ for $j$ if $d$ were changed, the other if $d$ were not changed.
Applying~\Cref{lem:Loose Bound of the GNMax} and~\Cref{lem:Composition of RDP} yields a sum of RDP values, each dependent on its specific sensitivity.
Since the sensitivity has a quadratic impact on the RDP costs of a Gaussian mechanism, votes for the same class are more expensive than votes for different classes (see~\Cref{lem:RDP Guarantee of the Gaussian Mechanism}).
%\todo{I do not really understand that. @Adam, can we talk about it and figure out how to describe that?}\adam{This is really because of 
Therefore, the RDP costs are the highest if all $u_d$ teachers trained on $d$ would consent on a class $j$ when trained on $d$ and would consent on class $j' \neq j$ if $d$ would be different.
\end{proof}

%Next, we consider the privacy evaluation for PATE.
%Concerning the 
To perform the privacy analysis in the framework of PATE, let $N$ and $N'$ be the numbers of sensitive data points and the number of the upsampled data points, respectively.
Then we can define the relative upsampling of training data as $u \coloneqq \nicefrac{N'}{N}$.
Since we keep the number of data points per teacher constant, the number of teachers $k$ has to be scaled by $u$.
The remaining PATE hyperparameters: $\sigma$ (for GNMax from~\Cref{eq:gnmax}), $\sigma_T$, and $T$ (for Confident GNMax from~\Cref{eq:confident-gnmax}) are scaled by $u$ as well to achieve a comparable voting accuracy and privacy efficiency as for the standard (non-\personalized) PATE.

\begin{prop}[Weighting Sensitivity]
\label{Weighting Sensitivity}
Let $d^{(i)} \in \mathcal{D}$ be a sensitive data point learned by teacher $t_i \in \{t_1, \ldots, t_k\}$. Let $w_i$ be the weight to determine the influence of $t_i$ to votings. Then, the individual sensitivity of the weighting vote count, regarding $d^{(i)}$, is:
\begin{align}
    \Delta_{\mathrm{weighting}, d}^{(i)} = w_i \text{.}
\end{align}
\end{prop}
\begin{proof}
In weighting PATE, every data point only influences one teacher. Therefore, on neighboring datasets, every vote count might change by the corresponding teacher's weight $w_i$.
\end{proof}

Note that the weighting approach does not change PATE hyperparameters ($\sigma$, $\sigma_T$, and $T$)\out.
Nonetheless, sensitive data has to be grouped budget-wise before being provided to the teachers\out.
The teachers are then given weights according to the budgets \st all weights sum up to the number of teachers $k$.

\subsubsection{Privacy Bounds}
Based on the mechanisms' sensitivity, we can formulate the loose bound of our \personalized aggregation mechanisms as follows:
\begin{thm}[Individual Loose Bound]
\label{thm:Individual Loose Bound}
Let $M$ be an \personalized GNMax aggregator with noise scale $\sigma \in \mathbb{R}_+$. Let further $d \in \mathcal{D}$ be any data point, and $\Delta_{M, d}$ be the individual sensitivity of $M$'s \personalized vote count regarding $d$. Then, $M$ satisfies an individual $(\alpha, (\Delta_{M, d})^2 \cdot \nicefrac{\alpha}{\sigma^2})$-RDP regarding $d$ for all $\alpha \in \mathbb{R}_+ \setminus \{1\}$.
\end{thm}
\begin{proof}
\Personalized GNMax aggregators can be considered as the composition of all classes' vote counts regarding each data point.
Only two of them can be changed at the same time on neighboring datasets.
Thus, the two Gaussian mechanisms with an individual sensitivity per data point are composed.
Therefore, the claimed RDP guarantee is achieved by using \Cref{lem:RDP Guarantee of the Gaussian Mechanism} on privacy guarantees of Gaussian mechanisms, and \Cref{lem:Composition of RDP} from \Cref{Appendix_Background} on composition.
Note that in the upsampling variant, more than two vote counts can be changed.
The worst case occurs if all teachers affected by data point $d$ change the same vote counts.
This is because votes for that same class are more expensive than votes for different classes (see the proof of \Cref{Upsampling Sensitivity}).
\end{proof}

We can also compute the data-dependent tight bound (Lemma~\ref{lem:Tight Bound of the GNMax} in~\Cref{Appendix_Background}) for our \personalized PATE variants.
PATE's calculation of the tight bound builds on the loose bound, and is calibrated for a sensitivity of 1 for the specified noise scale $\sigma$.
However, the sensitivity and the noise scale applied by PATE are related.
Therefore, when providing a different sensitivity than 1 to the tight bound calculation, it suffices to re-scale $\sigma$ according to that sensitivity. Our sensitivity values directly correspond to the parameters of our variants of PATE (upsampling duplication factors, participation frequencies in vanishing, and teachers' weights in the weighting method).
%\todo{@Adam, can you proof-read this paragraph for correctness?}\adam{Done.}

%As for the non-\personalized GNMax aggregator, the data-dependent tight bound (see Lemma~\ref{lem:Tight Bound of the GNMax} in~\Cref{Appendix_Background}) can be applied to the \personalized GNMax aggregators where all data sharing the same individual sensitivity also share the same tight bound.
%The tight bound builds up on the loose bound and assumes a sensitivity equal to one.
%Since sensitivity and noise scale are invariantly related to each other as the following corollary indicates, the tight bound can be applied using the noise scale relative to the individual sensitivity \st the latter equals one.
\begin{cor}[Scaling Invariance of the Individual Loose Bound]
\label{cor:Scaling Invariance of the Individual Loose Bound}
Let $c \in \mathbb{R}_+$ be any positive scalar. Let $M$ be an \personalized GNMax aggregator with noise scale $\sigma \in \mathbb{R}_+$ and an individual sensitivity $\Delta_{M, d} \in \mathbb{R}_+$ for some data point $d \in \mathcal{D}$. Furthermore, let $\tilde{M}$ be another \personalized GNMax aggregator with noise scale $\tilde{\sigma} = c \cdot \sigma$ and individual sensitivity $\Delta_{\tilde{M}, d} = c \cdot \Delta_{M, d}$ regarding $d$. Then, $M$ and $M'$ have the same individual loose bound regarding $d$ for any $\alpha \in \mathbb{R}_+ \setminus \{1\}$.
\end{cor}
\begin{proof}
Fix $\alpha \in \mathbb{R}_+ \setminus \{1\}$. $M, \tilde{M}$ satisfy individual $(\alpha, \varepsilon)$- and $(\alpha, \tilde{\varepsilon})$-RDP, respectively, regarding $d$. The equality of $\varepsilon$  and $\tilde{\varepsilon}$ is verified by direct computation as follows:
\begin{align}
\begin{split}
    \tilde{\varepsilon} & \coloneqq \left(\Delta_{\tilde{M}, d}\right)^2 \cdot \nicefrac{\alpha}{\tilde{\sigma}^2} \\
    & = \left(c \cdot \Delta_{M, d}\right)^2 \cdot \nicefrac{\alpha}{\left(c\cdot\sigma\right)^2} \\
    & = c^2 \cdot \left(\Delta_{M, d}\right)^2 \cdot \nicefrac{\alpha}{c^2 \cdot \sigma^2} \\
    & = \left(\Delta_{M, d}\right)^2 \cdot \nicefrac{\alpha}{\sigma^2} \\
    & \eqqcolon \varepsilon
\end{split}
\end{align}
\end{proof}

Note that all data points from the same privacy group share the same sensitivity, and, thereby, also have the same tight bound.

\section{Experimental Setup}
\label{sec:Experimental Setup}

In this section, we describe the setup for the empirical evaluation of our \personalized PATE variants.
%\new{Next to upsampling and standard weighting, we also include the variant of the weighting approach where some teachers have zero weight in some votings (vanishing).}
Over all experiments, we use the Confident-GNMax algorithm from~\citep{pate_2018}, where the privacy protection is ensured by Gaussian noise within PATE, and labels are only produced if a consensus among the teachers is reached.
To isolate the performance-gain of our \personalized PATE variants, we do not perform additional methods to improve utility of the student model from previous PATE papers, such as virtual adversarial training~\citep{vat} or MixMatch~\citep{berthelot2019mixmatch}.
%---semi-supervised techniques that are applied to PATE improve the accuracy of the PATE student model. However, we contribute improvements to the data generation and aggregation parts of PATE, hence we isolate the performance of the ensemble of models that 
Foregoing these methods allows us for a direct and more precise comparison between standard PATE and our new variants of the framework.
However, as a consequence, our reported student accuracies cannot be compared to the accuracies reported in~\citep{pate_2018}.
\new{Therefore, as a baseline to compare our individualized variants, we implement standard PATE within our framework following~\citep{pate_2018}.}
Our framework includes Gaussian PATE (GNMax, Confident-GNMax, Interactive-GNMax), our proposed \personalized variants, and the support for experimentation is implemented using \textsf{Python} (version 3.8)~\citep{python}.
\new{
Our code can be accessed online.\footnote{ 
\ifthenelse{\boolean{review}}{
\url{https://github.com/secret-pets-submitter/individualized-pate-pets-submission-}
}
%else
{
%\url{https://github.com/secret-pets-submitter/individualized-pate-pets-submission-}
\url{https://github.com/fraboeni/individualized-pate}
}}}

% \subsection{Framework}
% \label{sec:Framework}
% We present a new framework that comprises Gaussian PATE (GNMax, Confident-GNMax, Interactive-GNMax), our proposed \personalized variants, and the support for experimentation. Except for the tight bound and corresponding helper functions which we reuse from~\citep{pate-code}, the whole framework is implemented from scratch using \textsf{Python} (version 3.8)~\citep{python} .

\subsection{Datasets and Models}
\label{sec:Datasets and Models}
We conduct the experiments presented in this section on the MNIST \citep{mnist} and the Adult income dataset \citep{adult}.
MNIST consists of  $70,000$ ($28 \times 28$)-pixel gray-scale images depicting handwritten digits for classification. 
We scale the pixel values of all images to range $[0,1]$.
The Adult income dataset contains $48,842$ tabular data points from the US census of the year 1994. The corresponding classification task is to predict if the yearly income of a person represented in the data is greater than $\$50k$. 
As a pre-processing of the data, we remove $3,620$ damaged data points from the dataset and transformed categorical features into numerical values. 
Furthermore, we normalize these numerical values to the range of zero to one.

To train the teacher and student models on MNIST, we use a simple convolutional neural network (CNN) architecture taken from~\citep{convnet} (see \Cref{tab:convnet}). All weights in the output layer are initialized by values randomly sampled from the Glorot uniform distribution, whereas all other weights are sampled from the He uniform distribution.
Optimization is performed using the Adam optimizer and categorical cross-entropy loss. All other parameters are set according to the default values from \textsf{TensorFlow} (version 2.4.1). 

For the Adult income dataset, the teacher and student models are implemented as random forest models from the \textsf{scikit-learn} library. Each random forest consists of $100$ decision trees.
Otherwise, the default parameters of the library are applied.

\begin{table}[ht]
\centering
\begin{tabular}{|c|c|c|c|}
\hline
layer & type of layer & parameters & activation \\
\hline
1 & convolutional & $32$ $(3, 3)$-kernels & ReLU \\
2 & batch normalization & - & - \\
3 & max pooling & size $(2, 2)$ & - \\
4 & flatten & - & - \\
5 & fully connected & $100$ nodes & ReLU \\
6 & batch normalization & - & - \\
7 & fully connected & $10$ nodes & softmax \\
\hline
\end{tabular}
\caption{CNN-architecture for MNIST.}
\label{tab:convnet}
\end{table}
\vspace{-0.7cm}
Since, as done in standard PATE~\cite{pate_2017}, every teacher is provided with only $240$ data points for training on MNIST, we apply a custom data augmentation within each individual teacher and the student to improve model performances.
Therefore, each data point within one model's training data is randomly rotated by up to $\pm 7.5 \degree$ and randomly shifted by up to $7\%$ both, in horizontal and vertical directions to make a larger training dataset for that model. 
This data augmentation does not influence the privacy costs since we augment the data points only within their respective model's training dataset and not over different datasets.
As a consequence, augmented data points are solely used to train the same model as their original data point.
Since PATE is already based on the assumption that each single data point can completely determine the behavior of a corresponding teacher model (\cf \Cref{lem:Loose Bound of the GNMax} and \Cref{lem:Tight Bound of the GNMax}), no additional DP costs are incurred by augmenting datapoints.
For the experiments based on the Adult income data, no data augmentation is applied since it does not yield any performance benefits.

\subsection{Evaluation Metrics}
\label{sub:evaluation_metrics}

To measure the utility of the different PATE variants and privacy budget distributions, we mainly track \new{three} metrics. First, we count the \emph{number of produced labels} until any of the specified privacy budgets is exhausted. Second, we measure the \emph{accuracy of the student model} trained on that resulting labeled data.
Additionally, we also analyze the \textit{accuracy of the generated labels} ("voting accuracy").
As baselines to compare our \personalized methods to, we conduct experiments with standard non-\personalized PATE and Confident-GNMax using as the dataset-wide~$\varepsilon$ the \emph{minimum} privacy budget encountered in the sensitive data.
This has to be done in order not to violate any training data point's privacy requirements.
%The maximum and average baselines constitute the cases of all sensitive data having the maximum or the average budget among individual privacy budgets in \personalized experiments and thereby provide natural reference values.

\subsection{PATE Experiments}
\label{sec:PATE Experiments}
To experimentally evaluate our \new{two} novel \personalized PATE variants, we carry out the empirical analysis in four steps. 
(1) At first, the complete dataset is randomly divided into private, public, and test partitions.
Note that for increased randomization over the experiments, we do not rely on the standard train-test split in MNIST, but instead combine all $70,000$ data points and then partition the dataset.
The sizes of these partitions as well as general parameters for Confident-GNMax and its \personalized variants on both datasets are described in~\Cref{tab:parameters}, \new{where we follow the setup from PATE~\cite{pate_2018} in terms of the number of data points per set.} 
%These parameters were the basis for all \personalized experiments as they applied the Confident-GNMax approach as well. 
The parameters are adopted in the upsampling mechanism\out~so that teacher accuracies and voting accuracies align with those of weighting and non-\personalized experiments. (2) Privacy budgets are randomly assigned to the private data according to a given privacy budget distribution. Afterward, the data is allocated to the corresponding teacher models for training. (3) The trained teachers are used to produce labels in the voting process. Aggregation of the teacher votes is conducted according to the PATE variant under evaluation. As a baseline to evaluate our \new{two} variants, we use the standard non-\personalized Confident-GNMax. After every voting, the current accumulated RDP costs of data points are computed and stored group-wise. For \emph{upsampling}, all data points that share the same number of duplicates have the same privacy costs whereas in\out~\emph{weighting}, all data points that are learned by teachers of the same weight exhibit the same privacy costs. We consider all-natural RDP $\alpha$ values from $2$ to $50$. These RDP costs are transformed into standard DP costs by taking the best $\alpha$ at that point of the voting. After $2,000$ produced labels, the voting process is terminated since we observe that all experiments could exhaust their privacy budget within that number. Tracking privacy costs above the actual budget exhaustion up to the fixed number of $2,000$ generated labels is done to compare the privacy costs are spent over many votings. 
(4) The student model is trained on the labeled data that the respective teacher ensemble produced until any private data point's privacy budget is exceeded.
To get more reliable results, we average our measurements in all following experiments over multiple runs for the same parameters with the different random initialization, and for data shuffling and noise invoked.

\begin{table*}[t]
\centering
{
\begin{tabular}{|l|c|cccc|c|c|c|c|}
\hline
dataset & \# teachers & \# data & private & public & test & $\sigma_T$ & $\sigma$ & $T$ & $\delta$ \\
\hline
MNIST & 250 && 60,000 & 9,000 & 1,000 & 150 & 40 & 200 & $10^{-5}$ \\
Adult & 250 && 37,222 & 7,000 & 1,000 & 200 & 40 & 300 & $10^{-5}$ \\
\hline
\end{tabular}
}
\caption{\textbf{PATE Parameters}. Parameters used in the experiments for the Confident-GNMax on the MNIST and Adult income datasets.
$\sigma$ is the standard deviation of the noise-induced to the label aggregation.
$\sigma_T$ specifies the standard deviation of noise used to check if the teachers have a consensus given the threshold $T$.
}
\label{tab:parameters}
\end{table*}

\subsubsection{Uniform Assignment of Privacy Budgets}
\label{sub:exp_random_assignment}
We conduct our first set of experiments on both datasets with various privacy budget distributions. \new{We use two privacy groups in the experiments reported in this section as a micro-analysis to clearly show the differences between our variants, and to avoid a combinatorial explosion of privacy budgets and distributions being depicted.} We assign one of two different budgets (a higher and a lower budget) to every data point in the private dataset at random.
We vary the ratio of data points having the higher budget among $25\%$, $50\%$, and $75\%$.
The lower budget is set to $\log 2 \approx 0.69$ over all experiments while we assign the higher budget from $\log 4$, $\log 8$, and $\log 16$. Using logarithmic values provides a more intuitive comparison among the privacy budgets since the formulation of DP (\Cref{df:Differential Privacy}) uses $\exp(\varepsilon)$.
Hence, an $\varepsilon = \log s$ for any real $s \geq 1$ is half of a privacy budget $\varepsilon' = \log 2s$.
For example with our chosen budgets, a budget of $\varepsilon = \log 8$ is four times as high as a budget of $\varepsilon' = \log 2$.
The data and resulting labels that are produced until any data point's privacy budget is exhausted are used to train the student models.

\subsubsection{Non-Uniform Assignment of Privacy Budgets}
\label{sub:exp_class_assignment}
In the previous experiment, we assign data points randomly to the given privacy groups and their respective privacy budgets.
However, this might not necessarily reflect real-world use-cases where individuals' privacy requirements can correlate with their characteristics, for example which class they belong to.
In \personalized PATE, a data point's privacy budget determines how much information that data point can contribute to the voting.
As a consequence, when individuals with specific characteristics, or individuals from a specific class have much higher or much lower privacy requirements than other individuals, this can introduce biases to the generated labels, and thereby, also to the student model.

In this experiment, we, therefore, evaluate how the performance of our \personalized PATE variants is influenced when the privacy budget distributions vary significantly between different classes.
To do so, we use the Adults income dataset, which has an unbalanced class distribution (incomes lower than $\$50k$ make $75.2\%$ of the dataset).
We assign the higher privacy budget solely to the underrepresented high-income class to determine to what extent this shifts the trained student model's predictions.
The higher privacy budgets are again set to $\log 4$, $\log 8$, and $\log 16$ , the lower budget to $\log 2$.
We vary the ratio of data in the underrepresented class that receives the higher privacy budget among $25\%$, $50\%$, $75\%$, and $100\%$.
To produce more reliable results for each budget and ratio combination, we train ten teacher ensembles, use each of them for five voting processes, and report the average.
The data and corresponding labels that are produced until any data point's privacy budget is exceeded represent the student model's training data.

We restrict ourselves to the upsampling mechanism for our evaluation, as in\out~weighting PATE, teachers are trained on data points with the same privacy budget.
Since, in this experiment, we assign privacy budgets according to the classes, most teachers would be trained on data from solely one class.
This would result in poor teacher performance.

% In our second set of experiments, we, therefore, set out to investigate whether the performance of the PATE algorithm could be improved by increasing the privacy budget for only small parts of the data---if that data is expressive, \ie from the underrepresented class.
% If so, larger parts of the data could be awarded high privacy protection, while achieving improved model utility.

% In contrast to the previously described experiments, only the upsampling approach is used for these experiments.
% This is due to the fact that for vanishing and weighting PATE, only data points with the same budget are given to the same teacher(s).
% But then, either the teachers are exclusively trained on points of the same class or those of the more frequent class need to have the higher budget as well.
% In the first case, teachers would perform poorly while the second case would not fit our intended analysis.

\section{Empirical Results}
\label{sec:Results}
%This section presents the results of our experiments described in the previous section.
We present the quantitative effects of \personalization in PATE based on the number of produced labels and the accuracy of the generated student models (as our evaluation metrics described in~\Cref{sub:evaluation_metrics}).

\subsection{Advantage of Individualization}
\label{sec:Advantage of Personalization}

To better understand how much privacy the generation of labels consumes on both data groups (lower and higher privacy), we track the privacy costs over the course of generating $2,000$ labels on both datasets for our \out novel PATE mechanisms.
\new{\Cref{fig:cost_history_mnist} showcases the continuous privacy costs of generating labels for the MNIST dataset using the upsampling mechanism over the respective data groups.
Therefore, $50$\% of the data points (randomly chosen) are assigned the lower privacy budget of $\log 2$.
The remaining $50$\% are assigned $\log 8$.
As a baseline, we plot the continuous privacy costs for standard PATE.
}
\out
To evaluate the number of labels that can actually be generated for the given privacy budget distribution, we have to count how many labels are returned before any data point's privacy budget is exceeded.
In \Cref{fig:cost_history_mnist}, this corresponds to the moment when either the lower costs reach the lower budget or the higher costs reach the higher budget, whatever happens first.
\new{In the setup depicted in \Cref{fig:cost_history_mnist}, our individualized PATE is able to generate more than three times the number of labels generated by standard PATE (890 vs 257).}

\begin{figure*}[ht]
\makebox[\textwidth][c]{\includegraphics[width=0.7\textwidth, trim=0cm 0cm 0cm 0cm]
{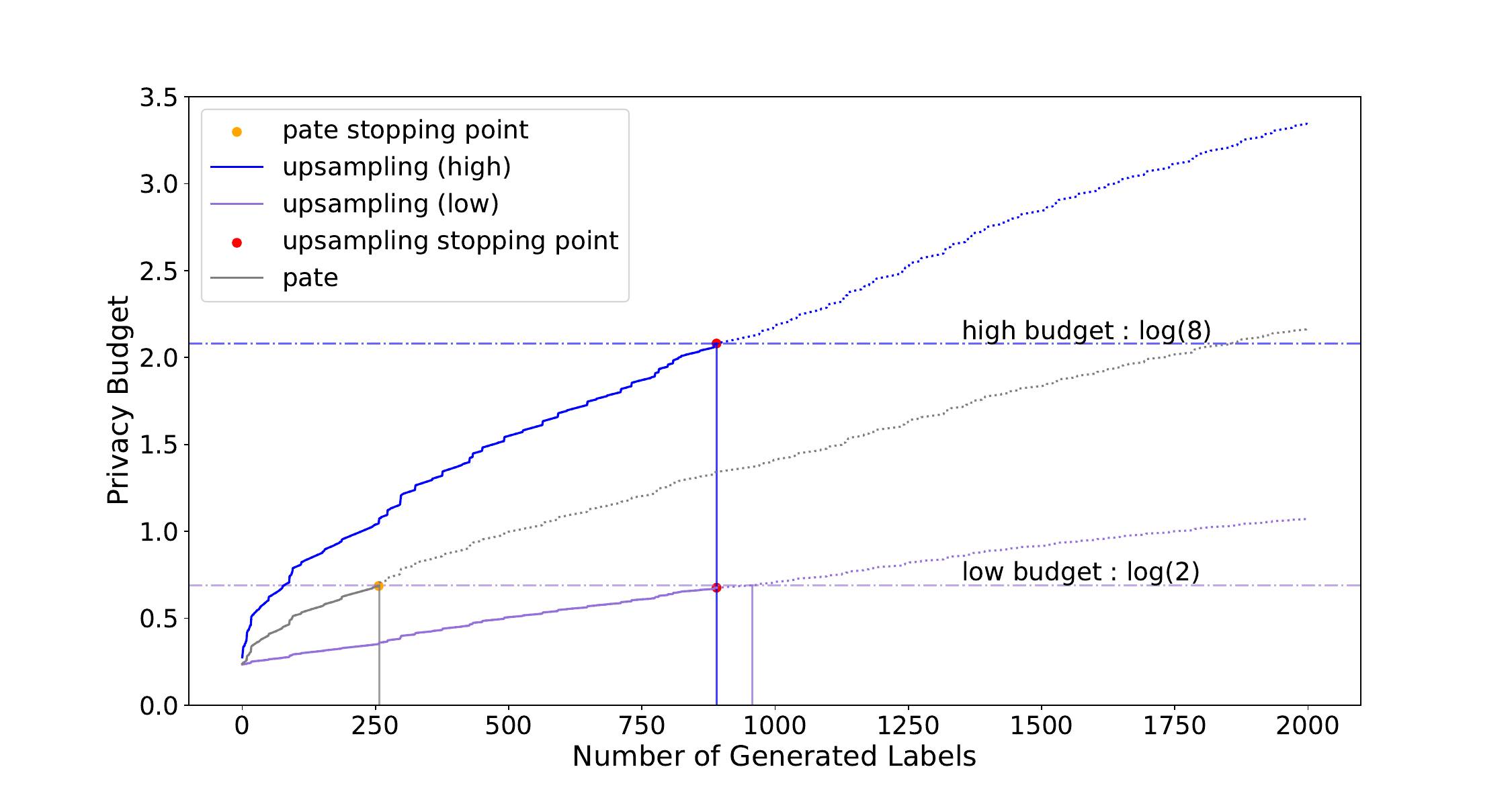}}
\vspace{-0.8cm}
\caption{\textbf{\Personalized upsampling outperforms standard PATE.} We compare the number of generated labels for given privacy budgets: $\log 2$ and $\log 8$. The lines for upsampling (high) and (low) and PATE (denoted as pate) represent the privacy costs during the generation of the first $2,000$ labels \new{for the MNIST dataset}. Standard PATE generates only 257 labels before the low privacy budget of $\log 2$ is reached (pate stopping point). The upsampling method exhausts its privacy budgets much later than PATE since it takes advantage of different privacy budgets. The upsampling (high) exhausts its budget (of $\log 8$) at 890 labels (upsampling stopping point). The upsampling (low) exhausts its budget (of $\log 2$) at 957 labels. Overall, the upsampling method exhausts its budget (upsampling stopping point) at the minimum of the two privacy groups ($\log 2$ and $\log 8$), therefore returning 890 generated labels, which is more than 3 times of the number of labels returned by standard PATE. For the upsampling method, we have to adjust its parameters (the number of times the data points from different privacy groups are upsampled) so that the different privacy budgets are exhausted at approximately the same number of generated labels.
} 
\label{fig:cost_history_mnist}
\end{figure*}

% \iffalse
%     \caption{\textbf{Privacy Cost History (MNIST)}. Costs for generating the first $2,000$ labels. Results are averaged over five voting processes by ten teacher ensembles, each for different budget distributions and GNMax variants upsampling, vanishing, and weighting. Privacy costs and budgets are given in $(\varepsilon, \delta)$-DP for $\delta = 10^{-5}$. Ratios indicate the proportion of data with a higher budget. Costs are listed per group of data points sharing the same budget.}
% \fi

%and \Cref{fig:cost_history_adult} (\Cref{Additional Results}) depict the results over various \personalized privacy distributions for the MNIST and Adult income dataset, respectively.
%For comparison, we also plot the privacy costs for the standard (non-\personalized) PATE.

For more extensive results on \new{the privacy budget consumption for label generation on }MNIST and Adult, see \Cref{fig:cost_history_mnist_large} and \Cref{fig:cost_history_adult} in \Cref{Additional Results}, respectively.
The resulting numbers of produced labels over the experiments are shown in \Cref{tab:n_labels_mnist_short} for the MNIST dataset and in \Cref{tab:n_labels_adult} in the Appendix for Adult income.

\new{In the results, we observe several different trends.
First, when analyzing the lines corresponding to the privacy costs in \Cref{fig:cost_history_mnist_large} and \Cref{fig:cost_history_adult}, we find that both lines differ more, the more the individual budgets differ.
This effect also increases when the proportion of sensitive data with a higher budget decreases.
Second, with an increasing ratio of the higher budget, both costs grow slower, resulting in more generated labels, see \Cref{tab:n_labels_mnist_short}.}
%Another observation is that the cost histories of the uGNMax and the wGNMax almost behave equally for all budget distributions.
%In comparison to the minimum non-\personalized PATE baseline, upsampling and weighting produce more labels while vanishing produces significantly fewer labels.
Thereby, the utility advantage of our \personalized PATE over the non-\personalized standard variant becomes visible.
For half of the sensitive MNIST data having a budget of $\log 4$, $\log 8$, or $\log 16$ and the other half having $\log 2$, $492$, $890$, and $1239$ labels can be produced by our weighting mechanism, respectively, instead of $257$ in the case of non-\personalized PATE.
This leads to a student accuracies of $93.08\%$, $94.68\%$, and $96.32\%$ while non-\personalized PATE only achieves $88.7\%$.
Analogously, on Adult income, in the same privacy budget configuration, $203$, $349$, and $530$ labels can be produced, leading to student accuracies of $81.76\%$, $82.60\%$, and $82.84\%$, respectively for weighting.
Standard non-\personalized PATE, instead, produces $88$ labels so that the student only achieves an accuracy of $79.85\%$.
\new{For a detailed overview on the final students' accuracies for MNIST and Adult with the different \personalized variants and privacy budget distributions, see \Cref{tab:accuracies_mnist} and \Cref{tab:accuracies_adult} in \Cref{Additional Results}.}

%FYI  @Franziska Boenisch in the 2nd PATE paper their accuracy is around 91% but their epsilon is much larger, namely 5.
Our results are not directly comparable to those \citep{pate_2018} since, in contrast to their work, we do not apply virtual adversarial training but only use the public data.
Their final models' accuracy is $98.5\%$ on MNIST with $\varepsilon = 1.97$ while our student model never surpassed $97\%$ even for a privacy budget of $\varepsilon = \log 16 \approx 2.77$ on $75\%$ of the data.
We decided not to integrate the adversarial training method in order to study the pure effect of our \personalization and exclude any other effects on the resulting model utility.

However, our individualization still outperforms~\citep{pate_2018}  when it comes to the voting accuracy, \ie the proportion of correctly generated labels: 
%In \citep{pate_2018}, $286$ labels are produced on the MNIST dataset with $\varepsilon = 1.97$, while our non-\personalized PATE produces more than $1500$ labels for the same privacy budget (see gray line in \Cref{fig:cost_history_mnist} and when it would reach the privacy budget of roughly $1.97$ on the y-axis).
Our generated labels are more accurate ($\approx 97.7\%$) than theirs ($93.18\%$) when evaluating them against ground truth, which is partly due to the better accuracy of our teachers.
Our teachers achieve an average test accuracy of $90.2\%$ ($81.7\%$ on Adult) on average while theirs are at $83.86\%$ ($83.18\%$ on Adult).
%We also report the student accuracies for MNIST and Adults in~\Cref{tab:accuracies_mnist}, and \Cref{tab:accuracies_adult}, respectively in the \Cref{Additional Results}.

Note that the budget combinations $\log 4$ with $75\%$ and $\log 8$ with $25\%$ yield the same average privacy budget over the entire dataset. 
Nevertheless, the experiment on distribution $\log 4$ with $75\%$ yields more labels and higher accuracy than that on $\log 8$ with $25\%$, see \Cref{tab:n_labels_mnist_short}.
This might indicate that having a smaller gap between the lower and the higher privacy budget leads to increased performance and that it might be better to have more data points with slightly higher privacy budgets than a few data points with very high privacy budgets.

% \begin{table}[ht]
% \small
% \centering
% {
% \begin{tabular}{|c|rrr|rrr|rrr|}
% \hline
% higher & \multicolumn{3}{c|}{25\% ratio} & \multicolumn{3}{c|}{50\% ratio} &
% \multicolumn{3}{c|}{75\% ratio}\\
% budget in $\varepsilon$ & u & v & w & u & v & w & u & v & w\\
% \hline
% $\log$ 4 & 158 & \color{lightgray} 25 & 158 & 237 & \color{lightgray} 67 & 239 & 326 & 157 & 326\\
% \hline
% $\log$ 8 & 231 & \color{lightgray} 0 & 229 & 414 & \color{lightgray} 84 & 414 & 636 & 285 & 638\\
% \hline
% $\log$ 16 & 308 & \color{lightgray} 0 & 308 & 623 & 115 & 623 & 1,038 & 477 & 1,041\\
% \hline
% baseline & \multicolumn{4}{c}{} & 99 & \multicolumn{4}{c|}{}\\
% \bottomrule
% \end{tabular}
% }
% \caption{\textbf{Labels per \Personalization (MNIST)}. Number of produced labels. Results depict the average over five voting processes by ten teacher ensembles each for different budget distributions and GNMax variants (\textbf{u}psampling, \textbf{v}anishing, \textbf{w}eighting). Non-\personalized GNMax using the minimum budget serves as baseline. The voting accuracies of all experiments are $\approx 97.7\%$. Inherent inefficiencies of vanishing PATE led to exhaustion before any label was produced in two cases.}
% \label{tab:n_labels_mnist_short}
% \end{table}

\begin{table}[hb]
%\small
\small
\centering
{
\begin{tabular}{|c|rr|rr|rr|}
%\toprule
\hline
higher & \multicolumn{2}{c|}{25\% ratio} & \multicolumn{2}{c|}{50\% ratio} & \multicolumn{2}{c|}{75\% ratio}\\
\cline{2-7}
budget in $\varepsilon$ & U  & W & U  & W & U  & W \\
\hline
$\log$ 4 & 158 & \textbf{433} & 237 & \textbf{492} & 326 &  \textbf{564} \\
\hline
$\log$ 8 & 231 & \textbf{474} & 414 & \textbf{890} & 636 &  \textbf{1163} \\
\hline
$\log$ 16 & 308 & \textbf{648} & 623 & \textbf{1239} & 1038 & \textbf{1787} \\
\hline
\textit{baseline} & \multicolumn{2}{c}{} & \multicolumn{2}{c}{\textit{257}} & \multicolumn{2}{c|}{}\\
\hline
\end{tabular}
}
%%% Results with Vanishing
% {
% \begin{tabular}{|c|rrr|rrr|rrr|}
% %\toprule
% \hline
% higher & \multicolumn{3}{c|}{25\% ratio} & \multicolumn{3}{c|}{50\% ratio} & \multicolumn{3}{c|}{75\% ratio}\\
% \cline{2-10}
% budget in $\varepsilon$ & U & V & W & U & V & W & U & V & W \\
% \hline
% $\log$ 4 & 158 & 25 & \textbf{433} & 237 & 67 & \textbf{492} & 326 & 157 & \textbf{564} \\
% \hline
% $\log$ 8 & 231 & 22 & \textbf{474} & 414 & 84 & \textbf{890} & 636 & 285 & \textbf{1163} \\
% \hline
% $\log$ 16 & 308 & 28 & \textbf{648} & 623 & 115 & \textbf{1239} & 1038 & 477 & \textbf{1787} \\
% \hline
% \textit{baseline} & \multicolumn{4}{c}{} & {\textit{257}} & \multicolumn{4}{c|}{}\\
% \hline
% \end{tabular}
% }
\caption{\textbf{Number of labels returned by \personalized PATE (\textbf{\underline{U}}psampling and \textbf{\underline{W}}eighting).} Non-\personalized GNMax using the minimum budget of $\log$ 2 serves as baseline. 
%For each privacy group, we report the average over five votings and ten different teacher ensembles. 
The voting accuracies for all methods are $\approx 97.7\%$. %Inherent inefficiencies of vanishing PATE lead to exhaustion before any label is produced in two cases of 25\% ratio.
}
\label{tab:n_labels_mnist_short}
\end{table}

\iffalse
version before our reweighting algo:
    \begin{tabular}{|c|rrr|rrr|rrr|}
    %\toprule
    \hline
    higher & \multicolumn{3}{c|}{25\% ratio} & \multicolumn{3}{c|}{50\% ratio} & \multicolumn{3}{c|}{75\% ratio}\\
    \cline{2-10}
    budget in $\varepsilon$ & U & V & W & U & V & W & U & V & W \\
    \hline
    $\log$ 4 & \textbf{158} & 25 & \textbf{158} & 237 & 67 & \textbf{239} & \textbf{326} & 157 & \textbf{326} \\
    \hline
    $\log$ 8 & \textbf{231} & 0 & 229 & \textbf{414} & 84 & \textbf{414} & 636 & 285 & \textbf{638} \\
    \hline
    $\log$ 16 & \textbf{308} & 0 & \textbf{308} & \textbf{623} & 115 & \textbf{623} & 1,038 & 477 & \textbf{1,041} \\
    \hline
    \textit{baseline} & \multicolumn{4}{c}{} & \textit{99} & \multicolumn{4}{c|}{}\\
    \hline
    \end{tabular}
\fi

\subsection{Generated Labels as a Function of Privacy and Relative Group Size}
\label{sec:Scaling generated labels}

We observe that there are two main factors that allow the individualized PATE algorithm to increase the number of labels that are generated: the number of individuals that have a larger privacy budget, and the actual size of the non-minimum privacy budget. Either increasing the number of individuals that have a larger privacy budget or increasing the larger privacy budgets, allows our individualized PATE to incur a smaller privacy cost on the most privacy-conscious group. 

We run an experiment measuring the number of generated labels for a series of budget combinations $((1.,2.), (1.,3.) ...)$, and for the group distributions $((25\%, 75\%), (50\%, 50\%) , (75\%, 25\%))$. We find that the relationship between the contributions of these two is in fact linear. Scaling up the number of individuals that have a large privacy budget, while equivalently scaling down the privacy budget of that group, keeps the number of generated labels roughly equivalent; and vice-versa. We find this effect to be significant for both our upsampling and weighting mechanism. A more detailed analysis of how to select a scaling, given a privacy-ratio, when group size is fixed is given in~\Cref{fig:find_ratios}.

Note that we implement our individualization through the algorithms from~\Cref{sec:Personalized Extensions for PATE}.
Using a different approach to implement our variants would change the curve. 
Hence, \Cref{fig:weighted_ratios_of_label_counts} allows us to assess the selection of hyperparameters (upsampling factors\out~and teachers' weights) for our variants of PATE. 
The lower the curve-of-best-fit for an algorithm is, the better it is at utilizing differences in privacy budgets, to generate more labels. 
%\todo{question - I am using 25\% - 75\% to signify 25\% of the small class}
 
\begin{figure}[t]
\includegraphics[width=0.7\linewidth, trim=0cm 0cm 0cm 0cm]
{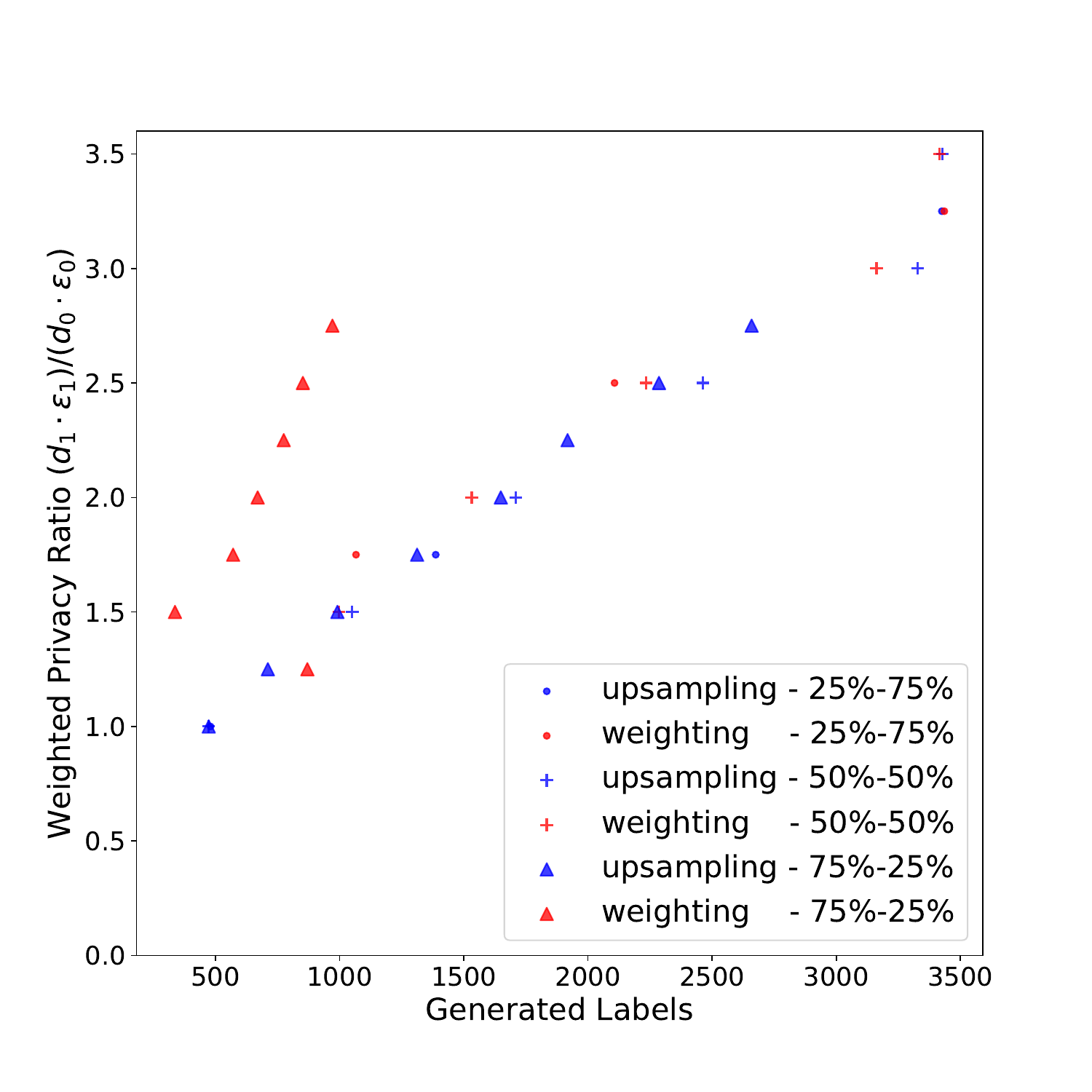}
\vspace{-0.7cm}
\caption{\textbf{Weighted ratio of generated labels}. The y-axis is the weighted privacy ratio, where the contribution of the privacy groups is scaled by the group size, for two different privacy groups this is $(\varepsilon_1 \cdot d_1) / (\varepsilon_0 \cdot d_0)$. There is a nearly-linear relationship between the number of generated labels, and the weighted privacy ratio. } 
\label{fig:weighted_ratios_of_label_counts}
\end{figure}   
%\vskip -0.2in
%\todo{@Roy: describe the y-axis, make the line of the best fit, make the labels larger, 
% reference appendix b
% don't use the word "line of best fit"
% }

\vspace{-0.1cm}
%\subsection{Non-Random Assignment of Privacy Budgets}
\subsection{Non-Uniform Privacy Budgets}
\label{sec:Compensating Underrepresented Data}

\begin{table*}[ht]
\centering
%\resizebox{\textwidth}{!}{%
\new{
\begin{tabular}{|c|rr|rr|rr|rr|}
\hline
higher & \multicolumn{2}{c|}{25\% ratio} & \multicolumn{2}{c|}{50\% ratio} & \multicolumn{2}{c|}{75\% ratio} & \multicolumn{2}{c|}{100\% ratio}\\
budget in $\varepsilon$ & low & high & low & high & low & high & low & high\\
\hline
$\log$ 4 & 95.93 & 36.77 & 93.11 & 45.93 & 90.13 & 54.74 & 86.24 & 63.39\\
\hline
$\log$ 8 & 93.25 & 45.91 & 86.82 & 62.25 & 80.57 & 72.61 & 77.91 & 77.36\\
\hline
$\log$ 16 & 90.42 & 54.00 & 80.74 & 72.68 & 76.68 & 79.42 & 73.82 &  83.59\\
\hline
\textit{baseline} & \multicolumn{8}{c|}{(98.01, 24.78)} \\
%\color{lightgray} baselines &&& \multicolumn{2}{c}{\color{lightgray} accuracy 86.18} && \multicolumn{2}{c}{\color{lightgray} precision 69.41} && \multicolumn{2}{c}{\color{lightgray} recall 48.82} &&\\
\hline
\end{tabular}
%}
\caption{
\new{\textbf{Per-group student accuracy (in $\%$) for unbalanced \personalization (Adult).} Results reported for \emph{low}-income (majority) and \emph{high}-income (underrepresented) class as an average over 50 different students trained through five voting processes by ten teacher ensembles using \textbf{upsampling}.
The ratios specify what proportion of the underrepresented class was assigned the higher privacy budget. 
The remaining data obtained a privacy budget of $\log 2$.
Non-\personalized experiments with all data points having a privacy budget of $\log 2$ serve as the baseline.
As the proportion of data points from the  underrepresented class that obtain a higher privacy budget (or their respective budget) increases, we observe an increase in student accuracy on this class.
At the same time, the student accuracy on the majority class decreases.
}
\label{tab:student_acc_skew}
}
}

\end{table*}

The experiment in this section serves to evaluate the influence of assigning a higher privacy budget only to (parts of) the underrepresented class in the Adult income dataset.
\new{
We analyse the effects on both the teachers and the student model and on the generated labels.
\Cref{tab:student_acc_skew} highlights how assigning higher privacy budgets to different proportions of the underrepresented class causes changes in the resulting student models' predictions.
We observe that with larger proportions of data from the underrepresented class that receive a high privacy budget, the resulting student model's accuracy on this underrepresented class increases. 
At the same time, the student model's accuracy on the majority class decreases significantly.
The same holds when increasing the privacy budget on the underrepresented class.
}
%The numbers of produced labels are reported in~\Cref{tab:n_labels_skew} in \Cref{Additional Results}.
%We observe similar relations and trends as for the experiments described in \Cref{sec:Advantage of Personalization}.
%However, when it comes to accuracies, precision, and recall values of the resulting models, we observe significant differences, see~\Cref{tab:voting_acc_prec_rec_skew}.
%One main observation is that the accuracies decrease (1) when the fraction of data points from the underrepresented class which receive the higher privacy budget increases, and (2) when their assigned privacy budgets increases.
%For all budget distributions, the precision (ratio of true predictions among all positive predictions) is smaller than that for non-\personalized experiments.
%It also decreases with increasing budgets in the underrepresented class.
%Conversely, the recall (ratio of true predictions among all positives) is higher than that for non-\personalized experiments and increases with increasing budgets.
\new{\Cref{tab:teacher_acc_skew} and \Cref{tab:voting_acc_skew} in \Cref{Additional Results} show similar trends for the average teacher and voting accuracies, respectively.
These observations indicate that the higher the privacy budget of the underrepresented class (\ie, the lower their privacy requirement), or the more data from the underrepresented class requires lower privacy protection, the more frequently that class gets predicted.
This highlights that an \personalized privacy budget assignment is able influence the model predictions, and thereby, to enforce and to mitigate biases in the resulting ML models.
}
\new{
When it comes to label generation for the non-uniform privacy-budget assignment (see \Cref{tab:n_labels_skew} in the \Cref{Additional Results}), we observe an increase in the number of generated labels depending on the fraction of underrepresented data that is assigned a higher privacy budget and the respective budgets.
We also compare the the number of labels generated in this setup with the number of labels generated in the random privacy budget assignment (see \Cref{tab:n_labels_adult} in \Cref{Additional Results}).
The rightmost column of \Cref{tab:n_labels_skew} (100\% of the underrepresented class obtain the higher privacy budget) can be directly compared with the leftmost column in  \Cref{tab:n_labels_adult} where 25\% of the overall data obtain a higher budget.
This is because the underrepresented class represents roughly 25\% of the data.
We observe that the non-uniform assignment of privacy budgets yields fewer labels.
Hence, we can conclude that even when the performance of PATE on the class that receives a higher privacy budget increases, the overall performance decreases.
}

\section{Discussion and Future Work}
\label{sec:Discussion and Future Work}

This section discusses results and implications of this work, and provides an outlook on possible future research directions.

\subsection{Improving Utility with Individualization}
\label{sub:discussion_utility}
Particularly in sensitive domains such as health care, applying high utility ML models is crucial.
This is because incorrect model decisions can have catastrophic consequences.
Since introducing DP into training often yields decreased ML model utility, many parties still entirely forego its adaptation within their sensitive ML applications, or they assign a very high privacy budget for all data points, which results in low privacy protection.

The introduction of our \personalized PATE yields multiple benefits in these scenarios. 
First of all, our methods allow integration of PATE in a system where individual data holders can choose what privacy level they want their data to be treated with.
That option alone might make individuals more willing to share their data, which would result in the availability of more training data for the ML models.
This, in turn, is known to have positive effects on the model utility when training with DP~\citep{tramer2020differentially}.
Additionally, our experiments highlight that our \personalized PATE variants yield more generated labels and higher student model utility than standard PATE which has to comply with the most strict privacy requirements encountered in the training dataset.
%For the MNIST dataset, we observe that when as little as $25\%$ of the data points specify a privacy budget of $\log 4$ instead of the minimum privacy budget of $\log 2$, more than $150\%$ of the original number of labels can be generated (\cf~\Cref{tab:n_labels_mnist_short}).
%When further increasing the privacy budget for that small data population or including more data points with that budget into the training, up to ten times the amount of labels could be generated.
%In general, depending on the privacy budget distributions, even obtaining more labels becomes possible.
%This is due to the fact that the data points with higher privacy budgets can contribute more information while privacy of data points with higher protection requirements can be spared.

\subsection{Comparison of our Variants}
\label{sub:discussion_comparison}
In the practical comparison of our \new{two} PATE variants, we see that the upsampling and weighting variants constantly outperform\out~standard PATE.
Additionally, both variants have different benefits and use-cases: 
The upsampling approach offers high flexibility in terms of individual privacy budget preferences. 
In theory, each data point could require a different privacy budget and would just have to be upscaled accordingly.
In practice, upsampling can increase the computational costs of PATE significantly since more training data is available and more teacher models need to be trained.
Moreover, the upsampling method cannot be used for distributed scenarios where the sensitive training data belong to different parties, such as hospitals that jointly want to train a student model based on their respective patients' data.
This is because, for upsampling, the sensitive data would have to be shared among the different parties which is usually restricted by privacy regulations.
Weighting is well suited for such distributed scenarios because each party can train their own teacher model and assign the weight according to the privacy requirements of their sensitive data.
However, to fully leverage the benefits of weighting, teachers must be trained on data points with the same privacy budgets, which reduces flexibility. 
It is possible to group together data points with different privacy budgets to one teacher within the weighting approach, but then this teacher's weight and corresponding privacy level must be set to comply with the strictest requirement among all its training data points.
Such an assignment  results in a waste of privacy budgets among all data points with higher privacy budgets within this teacher.

%This data can be used to (1) potentially increase the amount of training data per teacher, or (2) increase the number of teachers.
%(1) results in increased teacher accuracies, whereas (2) allows to use higher amounts of noise in the voting process, \ie lower privacy costs, without loss of voting accuracy.
%Both effects stand in direct relation since the more teachers are used, the less training data per teacher is available.
%To preserve the optimal point-to-teacher ratio, more teachers have to be trained on duplicated data.
%Additionally, since duplication of data has to be aligned to relations between this data's different privacy budgets, even points with the lowest budgets might have to be duplicated.
%Detailed insights are explained in \Cref{Implementation Details of the \Personalized Variants}.

One great advantage of our \personalized PATE variants is that\out~\new{their implementation can be configured such that all data points are able to exhaust their privacy budgets at roughly the same time.
This allows to fully make use of each individual data point's privacy budget, and thereby to fully leverage the sensitive training data in order to produce higher-utility ML models.}
%they can be combined during a voting process.
%For example towards the end of a label generation process, if not all privacy budgets are exceeded at the exact same time, the vanishing variant could be applied to exploit the remaining budgets while sparing the exhausted.

\subsection{Individualized Privacy and Biases}
\label{sub:discussion_outliers}
Our experiments on assigning higher privacy budgets to data points from one particular class highlight that \personalized DP guarantees can enforce or mitigate biases in the resulting ML models. 
More concretely, data with higher privacy budgets has a direct influence on what classes the student model predicts.
The higher a data point's privacy budget, the higher its influence on the model's prediction.
Therefore, whenever assigning \personalized privacy budgets, a thorough evaluation of the resulting ML models concerning biases and model fairness needs to be conducted. %needs be evaluated carefully from an angle of biases and model fairness.
Once such negative effects are detected, the privacy budgets of the respective (groups of) data points can be scaled down to reduce their influence.

\subsection{Outlook and Future Directions}
\label{sub:discussion_outlook}
%Based on the insights on the potential of \personalized privacy budgets being able to create or enforce biases in a dataset, future work should focus on combining the privacy and utility aspects of this work with considerations on biases fairness.
%This could be done by balancing information of different parts of the data distribution of interest to satisfy fairness aspects and ultimately to optimize utility.
%Such a balance could, for example, be achieved by not always exploiting the complete privacy budgets available for some data points.

Independent of \personalized privacy guarantees, further theoretical research on improving the tight bound analysis in PATE would be helpful to obtain more realistic estimates of the privacy costs during the voting process. 
The current analysis assumes that each data point can fully change its teacher model's prediction. 
In most scenarios, this assumption is, however, too strong. 
With a tighter estimate of a data point's influence, each vote consumes less privacy budget, and as a consequence, more labels can be produced.

Moreover, it would be of interest to study how applications of distributed PATE and similar frameworks, \eg~CaPC~\citep{choquette2021capc}, can benefit from our \personalized aggregation mechanisms.
Our mechanisms can be applied there to implement both individual data point privacy requirements, but also different "per-party" requirements.
What is more is that in particular for the weighting mechanism, these per-party privacy requirements could be extended to different weighting schemes taking into account, for example, the amount of training data a party holds, how diverse this data is, and how accurate their trained model predicts---assuming these properties can be determined without undermining the privacy guarantees of the system.
Such extensions can then support more meaningful cooperative ML model training and yield models of higher utility.

Finally, in this work, we focus on \personalized extensions of PATE.
Due to its structure, PATE is naturally suited to support different privacy budgets among its training data.
Also, data that does not require any privacy protection can directly be leveraged by the framework as public training data for the student model.
However, in the future, it would also be of interest to develop extensions of DP-SGD to support \personalized privacy guarantees within this framework. 
Such extensions could, for example, be implemented by sub-sampling the model's training data points with non-uniform probabilities according to their privacy budgets, or adding DP noise with different magnitudes to different data points' gradients.
%Finally, it will be interesting to study the behavior of our novel \personalized PATE variants further, also with more complex privacy budget distributions within the dataset.
%We see the biggest challenge for this part in the fine-tuning of the parameters and the combination of the mechanisms such that all privacy budgets among the data points can be completely exhausted for maximal utility. 

\section{Conclusion}
\label{sec:Conclusion}
Preserving privacy for the training data in ML is a crucial topic.
Often, this privacy is achieved at the cost of the final model's utility.
To improve the privacy-utility trade-off and to cater to the requirement encountered among
all data holders, we propose \new{two} novel variants for the PATE algorithm that allow for the use of \personalized privacy budgets among the data points.
We formally define our variants, conduct theoretical analyses of their privacy bounds, and experimentally evaluate their effect on PATE's utility for different datasets and different privacy budget distributions within them.
Our results show that through \personalized PATE, we are able to generate significantly more labels in comparison to standard PATE which has to comply with the highest privacy requirements encountered in its training dataset.
The increased amount of labels also translates into significant improvements in the student model's accuracy.
Our \personalized PATE variants are, therefore, able to reduce the loss of utility that is usually introduced by DP.

\begin{acks}
This work is supported by the German Federal Ministry of Education and Research (grant 16SV8463: WerteRadar).
\end{acks}

% \section{Acknowledgements}
% This work is supported by the German Federal Ministry of Education and Research (grant 16SV8463: WerteRadar).

\newpage
\bibliographystyle{ACM-Reference-Format}
\bibliography{references}
\newpage
\appendix % this marks appendices

\section{Additional Background on RDP and PATE}
\label{Appendix_Background}
This section supplements \Cref{sec:Background} by providing formalizations of Rényi divergence, RDP composition, and the tight bound.

\subsection*{Rényi Differential Privacy}
\begin{df}[\cf \citep{rdp}, Def. 3]%Rényi Divergence
\label{df:Rényi Divergence}
Let $P$ and $Q$ be two probability distributions over $\mathcal{D}$. Rényi divergence of order $\alpha \in \mathbb{R}_+ \setminus \{1\}$ for $P$ and $Q$ can be defined as:
\begin{align}
    \mathbb{D}_\alpha \left[P \parallel Q \right] \coloneqq \frac{1}{\alpha - 1} \cdot \log \underset{x \sim Q}{\mathbb{E}} \left[ \left( \frac{P \left( x \right)}{Q \left(x \right)} \right)^\alpha \right] \; ,
\end{align}
where $x \sim Q$ expresses that samples $x \in \mathcal{D}$ follow the probability distribution $Q$.
\end{df}

Composition under RDP can be expressed as:
\begin{lem}[\cf \citep{rdp}, Prop. 1]%Composition of RDP
\label{lem:Composition of RDP}
Let $\mathcal{R}_1, \mathcal{R}_2$ be arbitrary result spaces. Let further $M_1 \colon \mathcal{D}^* \rightarrow \mathcal{R}_1$, $M_2 \colon \mathcal{D}^* \rightarrow \mathcal{R}_2$ be mechanisms that satisfy $(\alpha, \varepsilon_1)$ and $(\alpha, \varepsilon_2)$-RDP, respectively. Then, the composition $M_3(D) \mapsto (M_1(D), M_2(D))$ satisfies $(\alpha, \varepsilon_1 + \varepsilon_2)$-RDP.
\end{lem}
Note that~\Cref{lem:Composition of RDP} also holds for adaptive sequential composition as shown in~\citep{rdp}.

RDP guarantees can be transformed into DP guarantees as follows:

\begin{lem}[\cf \citep{rdp}, Prop. 3]%RDP to DP
\label{lem:RDP to DP}
Let $M \colon \mathcal{D}^* \rightarrow \mathcal{R}$ be an $(\alpha, \varepsilon)$-RDP mechanism. Then, $M$ also satisfies $(\varepsilon', \delta)$-DP with
\begin{align}
    \varepsilon' = \varepsilon + \frac{\ln \nicefrac{1}{\delta}}{\alpha - 1}
\end{align}
for all $\delta \in (0, 1]$.
\end{lem}
%\Cref{lem:Composition of RDP} and \Cref{lem:RDP to DP} are proved in \citep{rdp}.
\Cref{lem:RDP to DP} is proved in~\citep{rdp}.

\subsection*{Tight Bound Privacy Analysis of PATE}
The tight bound for PATE can be defined as follows.
\begin{lem}[\cf \citep{pate_2018}, Thm. 6]%Tight Bound of the GNMax
\label{lem:Tight Bound of the GNMax}
Let $M$ simultaneously satisfy $(\alpha_1, \varepsilon_1)$-RDP and $(\alpha_2, \varepsilon_2)$-RDP. Both RDP bounds can be computed by applying the loose bound for two different alpha values. Suppose that $1 \geq q \geq \mathbb{P} [M(D) \neq j^*]$ holds for a likely teacher voting $j^*$. Additionally suppose that $\alpha \leq \alpha_1$ and $q \leq
\exp((\alpha_2 - 1) \cdot \varepsilon_2) / \left(\frac{\alpha_1}{\alpha_1 - 1} \cdot \frac{\alpha_2}{\alpha_2 - 1} \right)^{\alpha_2}$. Then, $M$ satisfies $(\alpha, \varepsilon)$-RDP for any neighboring dataset $D'$ of $D$ with
\begin{align}
    \varepsilon = \frac{1}{\alpha - 1} \cdot \log \left( (1 - q) \cdot A + q \cdot B \right) \, ,
\end{align}
where $A$ and $B$ are defined as follows:
\begin{eqnarray}
    A & \coloneqq & \left( \frac{1 - q}{1 - \left(q \cdot e^{\varepsilon_2} \right)^{\frac{\alpha_2 - 1}{\alpha_2}}} \right)^{\alpha - 1} \, ,\\
    B & \coloneqq & \left( \frac{e^{\varepsilon_1}}{q^{\frac{1}{\alpha_1 - 1}}} \right)^{\alpha - 1} \, .
\end{eqnarray}
\end{lem}

This holds since according to \citep{pate_2018}, Prop. 7, for a GNMax aggregator $M$ with parameter $\sigma$ and for any class $j^* \in \mathcal{Y}$ the following statement applies:
\begin{align}
\label{eq:Probability of Likely Class}
    \mathbb{P} \left[ M(D) \neq j^* \right] \leq \frac{1}{2} \sum_{j \neq j^*} \mathrm{erfc} \left( \frac{n_{j^*} - n_j}{2\sigma} \right)
\end{align}
where $\mathrm{erfc}(\cdot)$ denotes the complementary error function defined by:
\begin{align}
    \mathrm{erfc}(a) \coloneqq \frac{2}{\sqrt{\pi}} \int_a^\infty e^{ -t^2} \mathrm{d}t \; .
\end{align}
See~\citep{pate_2018} for the proofs.

\section{The Vanishing-Mechanism}
\label{sec:Vanishing}
Our \emph{vanishing} mechanism keeps the independent partitioning of the original PATE approach and implements \personalized privacy by having teachers participate in more or fewer votings according to their training data points' privacy budget.
Therefore, in vanishing, data points with the same privacy budget have to be allocated to the same teacher (s).
We call data points with the same privacy budget a \emph{privacy group} $g_j$.
Teachers trained on privacy groups with higher privacy requirements (lower budgets) contribute to fewer votings, whereas teachers in lower-requirement groups contribute to more votings.
We implement vanishing by randomly sampling teachers for participating in given voting according to their data points' privacy requirements.
To be able to apply the same magnitude of privacy noise for each voting, we make sure that the number of teachers sampled per voting stays constant, see \Cref{alg:vanishing}.
The vanishing mechanism is also visualized in \Cref{fig:PATE Schemes}c.

\begin{algorithm}
\caption{Select teacher models for voting in the \textbf{vanishing} method.}\label{alg:vanishing}

\SetKwInput{KwData}{Input}
\KwData{Privacy budget $\varepsilon_j$ for each privacy group $g_j$, $j\in{1,...,G}$, each teacher model $t_i$.}
\KwResult{Participation $s_i$ for each teacher $t_i$.}
\For{Each teacher $t_i$}
{
    $s_i \gets 0$\Comment*[r]{Initialize participation}
}
$\varepsilon_{max} \gets \max_{j=1}^{G} \varepsilon_j$\;
\SetAlgoLined
\For{Each privacy group $g_j$}{
$S \gets $ randomly select $\frac{\varepsilon_j}{\varepsilon_{max}}$ teachers from group $g_j$\;
\For{Each teacher $t_i$ in $S$}
{
    $s_i \gets 1$\Comment*[r]{Update participation}
}
}
\end{algorithm}

We call the resulting aggregation method \emph{vanishing GNMax (vGNMax)}. Its vote count mechanism can be defined as follows:
\begin{df}[Vanishing Vote Count]
\label{Vanishing Vote Count}
Let $t_i \colon \mathcal{X} \rightarrow \mathcal{Y}$ be the $i$-th out of $k \in \mathbb{N}$ teachers.
Let further $N \in \mathbb{N}$ be the number of sensitive data points and $m_i \in \{0, 1\}^N$ a mapping that describes which points are learned by $t_i$.
Moreover, let $s_i \in \left\{0, 1 \right\}$ be the current participation of $t_i$.
The vanishing vote count $\mathring{n} \colon \mathcal{Y} \times \mathcal{X} \rightarrow \mathbb{N}$ of any class $j \in \mathcal{Y}$ for any unlabeled public data point $x \in \mathcal{X}$ is defined as
\begin{align}
    \mathring{n}_j \left( x \right) \coloneqq \sum\limits_{i=1}^k s_i \cdot \mathbbm{1} \left(t_i (x) = j \right) \;.
\end{align}
\end{df}

% In the vanishing algorithm, we select the same number of teachers for each voting (to use the same privacy noise $\sigma$, see~\Cref{df:Gaussian Mechanism}, for each voting).
% \textbf{Vanishing Algorithm:}
% \begin{enumerate}
%     \item Select the same number of teachers for each voting (to use the same privacy noise $\sigma$, see~\Cref{df:Gaussian Mechanism}, for each voting).
%     \item Probability of sampling a teacher depends on its privacy budget. The higher the privacy budget of the teacher, the more often the teacher participates in voting. 
%     \item The teachers with the highest privacy budgets participate in each voting while teachers from the lower privacy budgets participate with frequency that is proportional to the privacy budget. 
% \end{enumerate}
For example, if there are 2 groups of teachers where the higher privacy budget is twice the lower privacy budget, then the teachers with the higher privacy budget always vote while the teachers from the lower privacy budget participate only in half of the votings, and the number of teachers for given voting is 3/4 of the total number of teachers. %The vanishing mechanism modifies the selection of teachers for a given voting, as presented in~\Cref{alg:vanishing}.

\subsection*{Privacy Analysis for Vanishing}
\begin{prop}[Vanishing Sensitivity]
\label{Vanishing Sensitivity}
Let $d^{(i)} \in \mathcal{D}$ be a sensitive data point learned by teacher $t_i \in \{t_1, \ldots, t_k\}$. Let $s \coloneqq (s_1, \ldots, s_{k}) \in \{0, 1\}^k$ be the selection of teachers that participate in the current voting. Then, the individual sensitivity of the vanishing vote count, regarding $d^{(i)}$, is:
\begin{align}
    \Delta_{\mathrm{vanishing}, d}^{(i)} = s_i \text{.}
\end{align}

\end{prop}
\begin{proof}
In vanishing PATE, every data point only influences the vote of one teacher, which, in the worst-case results in two vote counts being changed.
However, in contrast to non-\personalized PATE, privacy is only spent if the teacher corresponding to $d^{(i)}$ participates in the current voting.
\end{proof}

Note that the vanishing mechanism could potentially benefit from the privacy amplification by subsampling~\citep{subsampling2018}. However, how to combine the data-dependent RDP, as used in PATE, with the subsampling mechanism remains an open problem~\citep{privateKnn2020} which is outside of the scope of this work. 

The vanishing approach does not change PATE hyperparameters ($\sigma$, $\sigma_T$, and $T$), in contrast to the upsampling method.
The sensitive data has to be grouped budget-wise before being provided to the teachers.
The votes are scaled so that the total sum of the votes is equal to the total number of teachers.

\section{Implementation Details of the \Personalized Variants of PATE}
\label{Implementation Details of the Personalized Variants}
This section contains details on the implementation of the \personalized GNMax variants which are left out in~\Cref{sec:Personalized Extensions for PATE,sec:Experimental Setup} for the sake of brevity.

\subsection*{Hyperparameter Search}
%\todo{Rework after having decided which experiments go into the paper!}
\label{Aligning Budgets and Expenditures}

The goal of the practical implementation of our variants of PATE is to ensure that their parameters align.
Thus, the optimization of PATE hyperparameters, \ie number of teachers $k$, noise standard deviation for consensus $\sigma_T$, threshold for consensus $T$, and the noise standard deviation for label creation $\sigma$, have to be adjusted so that the different variants of PATE are comparable.

We show how the parameters used in variants of \personalized PATE: numbers of duplications for upsampling, participation frequencies for vanishing, and teacher weights for weighting, influence the individual loose bound through individual sensitivities. For example, the teachers' weights for the weighting scheme translate directly to the teacher's sensitivities, which are set for the privacy analysis. The same holds for the duplication factor in upsampling and the participation frequencies for vanishing.

The privacy costs depend not only on the loose bound but also on the data-dependent tight bound and on the currently optimal RDP order(s). Therefore, it does not suffice to set the individual parameters or sensitivities proportional to the individual budgets. To find adequate parameters, we conduct experiments to analyze the relation between individual sensitivities and resulting privacy costs in~\Cref{fig:find_ratios}. 
The figure reports results for sensitivities, which are set according to the duplication, participation, and weighting factors of PATE.
The goal is to relate individual sensitivities by adjusting parameters so that all privacy budgets exhaust approximately at the same time.
To enable comparisons among the different variants, we describe the parameters by corresponding individual sensitivities.
We randomly divide the sensitive data into two equally sized groups, one with higher and one with lower individual sensitivity.
The parameters are adjusted so that a ratio of $c$ to $1$ is achieved for individual sensitivities with each $c \in \{2, \ldots, 9\}$. 
We conduct this experiment on the MNIST dataset and train ten different ensembles.
Each ensemble is then used for five different voting processes.
We perform $4,000$ votings in each process to compare the different cost growths over time.

\begin{figure*}[ht]
\makebox[\textwidth][c]{\includegraphics[width=0.7\textwidth, trim=3cm 0cm 3cm 1cm]{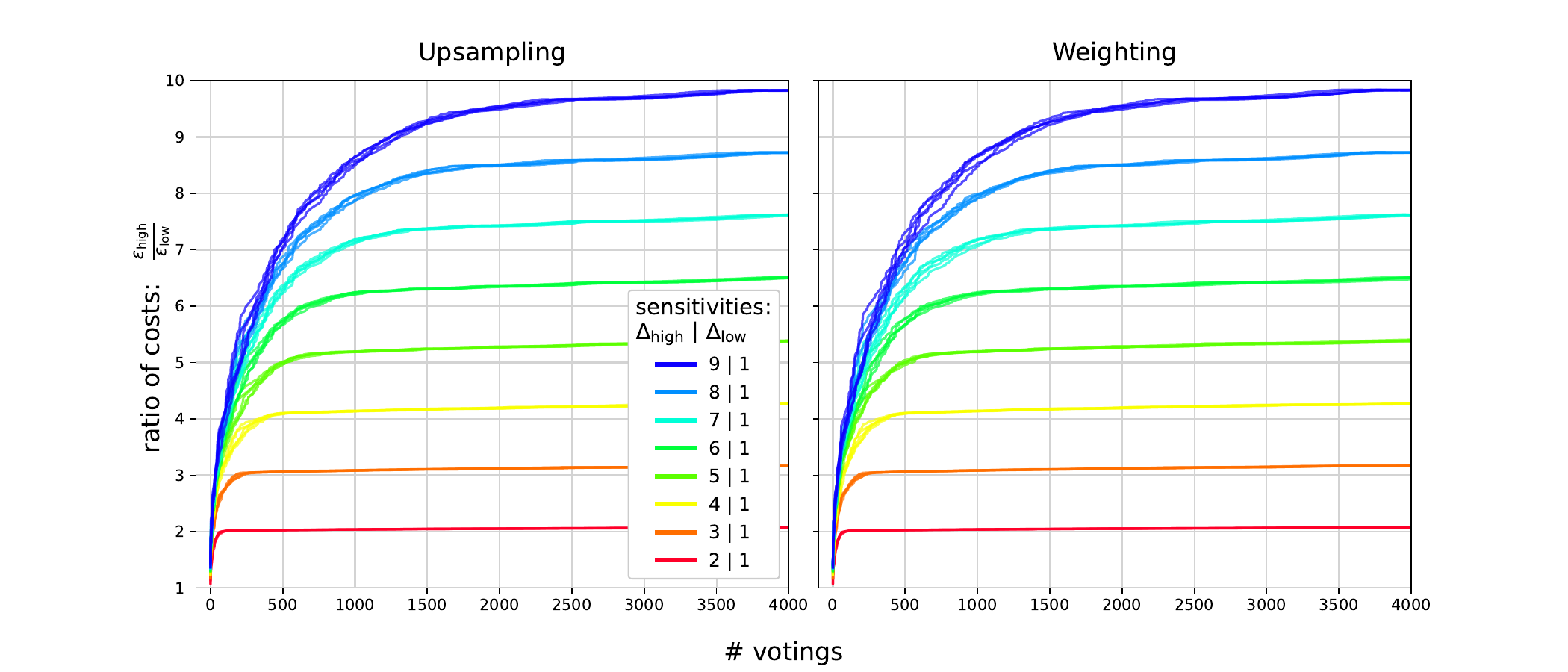}}
\caption{
\textbf{Tuning parameters for \personalized variants of PATE}.
Privacy cost relation between two equal-sized groups of sensitive data (high and low sensitivity) shown over $4,000$ votings on the MNIST dataset for both \personalized Confident-GNMax variants.
All costs are given in $(\varepsilon, \delta)$-DP for $\delta = 10^{-5}$.  Each of ten teacher ensembles is used to vote five times for all labels in the public dataset that is shuffled differently for each combination.
After some votings, the ratio of different costs almost remains constant at the ratio of corresponding sensitivities.
Note that the sensitivities shown are only proportions of sensitivities.
This means that for upsampling and weighting, the \emph{plain} sensitivities (that correspond to the number of duplications and the teachers' weights) are scaled while for vanishing, \emph{average} sensitivities (\ie participation frequencies), are scaled and depicted.
We use the figure to find what sensitivity values should be set for our two balanced (with the same number of data points) privacy groups. For a given ratio of $\varepsilon$-s and the desired number of votings (generated labels), we find the optimal ratios of sensitivities that directly correspond to the hyperparameters of our variants of PATE. For example, if the ratio of the privacy budgets $\frac{\varepsilon_{high}}{\varepsilon_{low}}$ is 3, then the corresponding ratio of sensitivities for the upsampling method should be 3 (the points from the lower privacy group are not duplicated - sampled once, while the points from the higher privacy group should be duplicated twice - with the total number of 3 points).  
}
\label{fig:find_ratios}
\end{figure*}

\subsection*{Details of Upsampling}
\label{Details of Upsampling}
Upsampling PATE extends the training data for teachers by duplicates.
\Cref{fig:find_ratios} shows that the ratio of privacy costs approaches the ratio of their corresponding budgets after some votings.
%the privacy costs of a data point having a higher privacy budget divided by the costs of another point having a lower budget approach a constant that equals the higher sensitivity divided by the lower sensitivity.
Therefore, numbers of duplicates should align to the relation of privacy budgets.
This can be achieved by initializing the numbers of duplicates as the different budgets and then scaling them up equally until each of them reaches an integer with some desired precision.
Note that a higher precision might lead to very high numbers of duplicates if not all budgets are multiples of each other as in our experiments.

\subsection*{Details of Vanishing}
\label{Details of Vanishing}
Vanishing PATE differentiates privacy on teacher-level by avoiding participation in some votings.
Therefore, sensitive data has to be grouped budget-wise and then be given to teachers \st all data points in a teacher have (almost) the same privacy budget.
Afterward, the participation frequencies have to be set according to the lowest budget of each teacher.
\Cref{fig:find_ratios} suggests that the frequencies corresponding to two different budgets should have a relation that is at least quadratic to the relation of their corresponding budgets.
In our experiments we used a relation that equals the relation of budgets to the power of four since the costs of data with different budgets are closer after a few votings before they approach constant ratios as for upsampling and weighting.
So we initialized frequencies to the corresponding budgets, squared them, and finally divided them by the highest frequency so that frequencies were probabilities and the highest one was $100\%$.

To be comparable to the other \personalized variants, the voting accuracy of vanishing PATE should be retained by decreasing the noise intensity according to the smaller number of voting teachers.
A weaker noise entails higher privacy costs for participating teachers' data.
Experiments showed that the privacy costs of all data is lower if the number of participating is stable over the whole voting process. Therefore, our implementation maintains a stable number of participating teachers by selecting random alternations that are changed periodically.
More precisely, randomly selected sets of teachers participate periodically in votings where the period aligns to their frequency and equal periods are shifted to achieve a stable number of participating teachers per voting.
After some votings, new sets of teachers with identical frequencies are randomly sampled to reduce the risk of biases that could be introduced into labels by cliques of teachers with similar knowledge.

\subsection*{Details of Weighting}
\label{Details of Weighting}
Following~\Cref{fig:find_ratios}, weights can be set to their corresponding budgets divided by the average budget.
Thus, all hyperparameters can remain unchanged while their optimization regarding the accuracy of teachers and voting still holds.

% \subsection*{Efficient Handling of \Personalized Privacy}
% \label{Efficient Handling}
% The consideration of different privacy budgets and costs increases the complexity of PATE both in time and space.
% To counteract this problem, privacy costs can be calculated and stored group-wise.
% That is all data points that have the same number of duplicates in the case of upsampling.
% In the vanishing approach, the data of all participating teachers have the same costs in a voting.
% However, costs have to be stored and accumulated teacher-wise.
% Usually, weighting PATE has the lowest complexity since all costs can be calculated, accumulated, and stored per group of data having the same budget.

\subsection*{Setting Parameters for \Personalized Variants} % These are indeed heuristics.
\label{sec:Adjustment of Privacy Personalization}
We show how to set the parameters of the \personalized variants of PATE so that different privacy budgets are exhausted at approximately the same time. 
\Cref{fig:find_ratios} visualizes the relation between the parameters of our \personalized Confident-GNMax variants and the resulting \personalized privacy costs according to tight bounds over time.
We observe that uGNMax and wGNMax behave very similarly and their cost ratios stay almost constant after a few votings. 
Contrary, vGNMax needs more votings to lower the gain of its cost ratio. 
For uGNMax and wGNMax, the cost ratio according to the tight bound seems to be approximately equal to the ratio of sensitivities, whereas the cost ratio approaches the square root of the ratio of sensitivities for vGNMax.
Therefore, in our experiments, we adjust the \personalization parameters (duplications, participation frequencies, and weights) so that the resulting sensitivities relate to the actual budgets.

\Eg let $\varepsilon_1, \varepsilon_2 \in \mathbb{R}_+$ be two DP budgets with $\varepsilon_2 = c \cdot \varepsilon_1$ for any $c > 1$.
Then, for the uGNMax, the duplications $u_2$ of points having the higher budget $\varepsilon_2$ are set to $u_2 \coloneqq c \cdot u_1$ where $u_1$ is the number of duplications for points having the lower budget.
For vGNMax, the participation frequency $s_1$ of teachers trained on points having the lower budget $\varepsilon_1$ is set to $s_1 \coloneqq \nicefrac{1}{c^2}$ while the frequency $s_2$ of teachers trained on points having the higher budget is always one\footnote{We set $s_1 \coloneqq \nicefrac{1}{c^4}$ for vGNMax in our experiments so that different privacy budgets exhausted approximately at the same time.}.
Finally, for the wGNMax, the weights of teachers are set to the corresponding budgets and then normalized so that the sum of all weights equals the number of teachers.
Thus, $w_1 \coloneqq \nicefrac{\varepsilon_1}{\overline{w}}$ and $w_2 \coloneqq \nicefrac{\varepsilon_2}{\overline{w}}$ where $\overline{w}$ is the average weight of all teachers.

%A comprehensive rationale on parameter settings for our \personalized variants as well as a description of further implementation details is left to~\Cref{Implementation Details of the \Personalized Variants}.

\section{Visualization of Methods}
We visualize the methods in~\Cref{fig:PATE Schemes}.
\renewcommand{\textfraction}{0.01} 
\begin{figure*}[h!]
\centering
 \includegraphics[width=\textwidth, trim={2cm 12.7cm 2cm 3.4cm},clip]{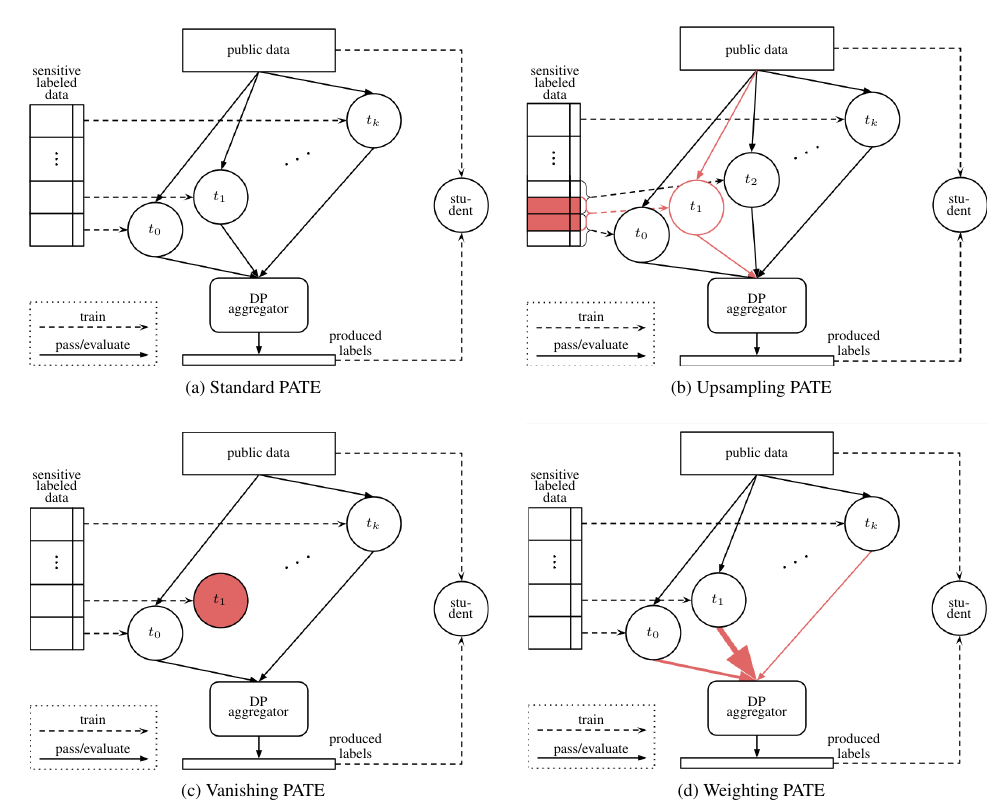}
\caption{
\textbf{PATE variants}.
Teachers $t_0, t_1, \ldots, t_k$ are trained on partitions of sensitive labeled data.
Afterward, public unlabeled data is given to the teacher ensemble whose votes are aggregated \st labels are produced.
Finally, the student is trained on the public data with produced labels.
The \personalized variants modify this procedure to individually adjust the influence of sensitive data on produced labels.
Sensitive data may be used to train multiple teachers (b).
Instead, teachers may avoid participating in some votings (c), or be weighted differently (d).
\label{fig:PATE Schemes}
}
\end{figure*}

\section{Additional Results}
\label{Additional Results}
Results of experiments on the Adult income dataset as well as more comprehensive results of MNIST experiments are presented on the following pages.

\begin{figure*}[ht]
\makebox[\textwidth][c]{\includegraphics[width=1\textwidth, trim=2.4cm 1.2cm 3.6cm 2cm]
{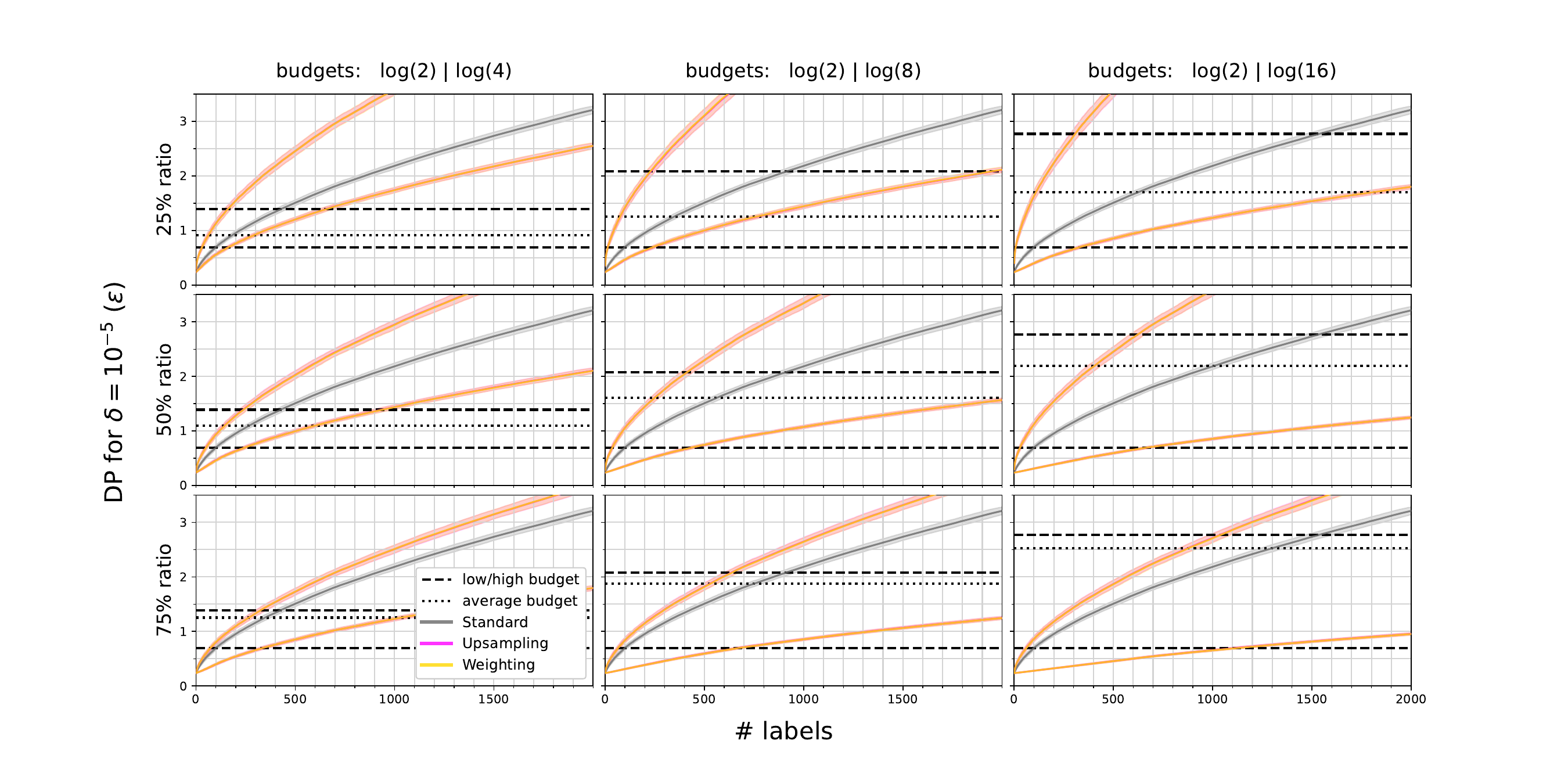}}
\caption{\textbf{Privacy cost history (MNIST)}. Costs for generating the first $2,000$ labels. Results are averaged over five voting processes by ten teacher ensembles, each for different budget distributions and GNMax variants upsampling\out~and weighting. Privacy costs and budgets are given in $(\varepsilon, \delta)$-DP for $\delta = 10^{-5}$. Ratios indicate the proportion of data with the higher budget. Costs are listed per group of data points sharing the same budget.}
\label{fig:cost_history_mnist_large}
\end{figure*}

\begin{figure*}[ht]
\makebox[\textwidth][c]{\includegraphics[width=1\textwidth, trim=2.4cm 1.2cm 3.6cm 2cm]
{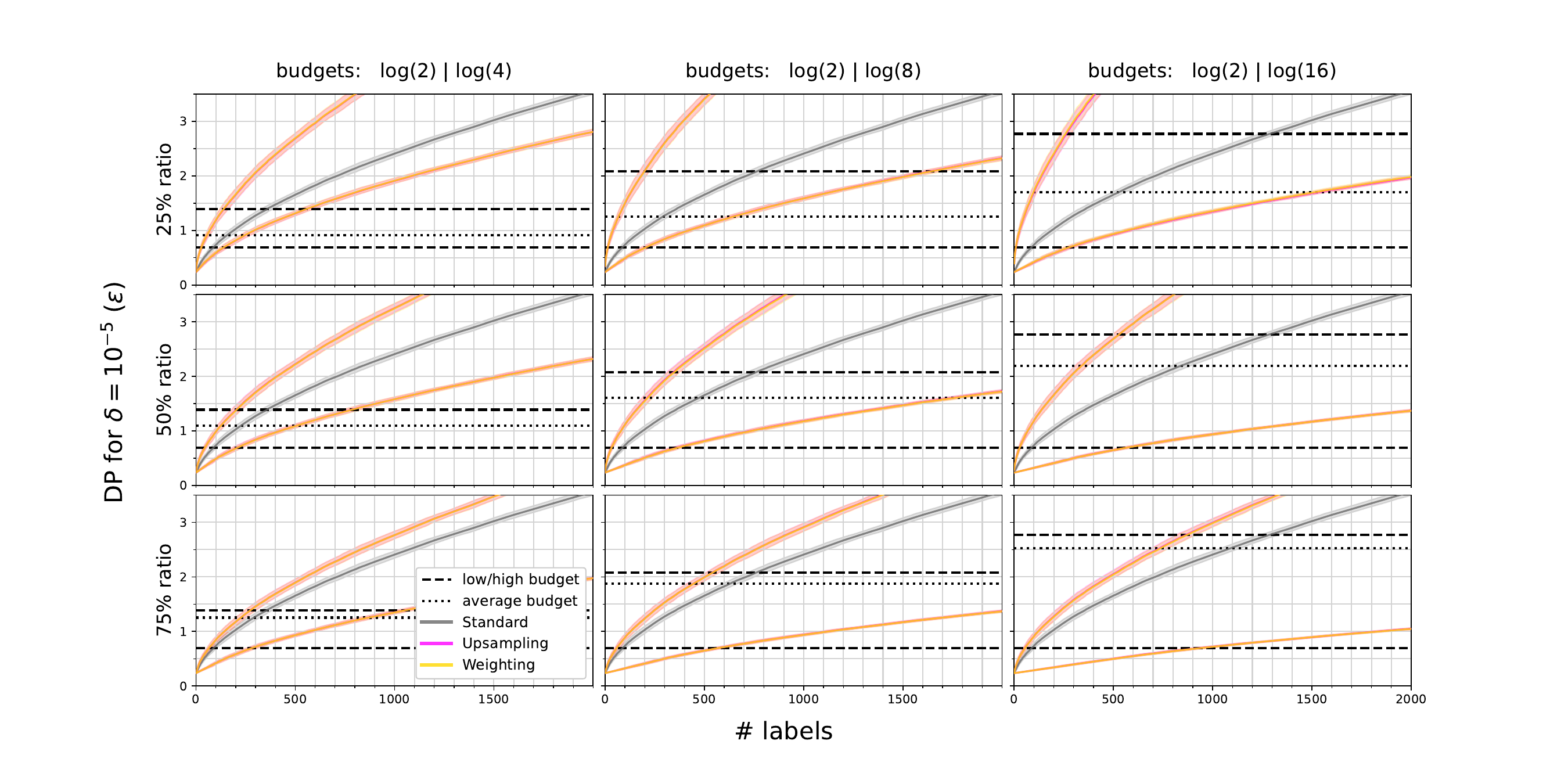}}
\caption{\textbf{Privacy cost history (Adult)}. Costs for generating the first $2,000$ labels. Results are averaged over five voting processes by ten teacher ensembles, each for different budget distributions and GNMax variants upsampling\out~and weighting. Privacy costs and budgets are given in $(\varepsilon, \delta)$-DP for $\delta = 10^{-5}$. Ratios indicate proportion of data with the data have the higher budget. Costs are given per group of points that share the same budget.}
\label{fig:cost_history_adult}
\end{figure*}

\colorlet{lightgray}{black}

\begin{table*}[ht]
\footnotesize
\centering
{
\begin{tabular}{|c|rr|rr|rr|}
%\toprule
\hline
higher & \multicolumn{2}{c|}{25\% ratio} & \multicolumn{2}{c|}{50\% ratio} & \multicolumn{2}{c|}{75\% ratio}\\
\cline{2-7}
budget in $\varepsilon$ & U  & W & U  & W & U  & W \\
\hline
$\log$ 4 & 140  & 139 & 202  & 203 & 272  & 273 \\
\hline
$\log$ 8 & 198  & 198 & 346  & 349 & 541  & 543 \\
\hline
$\log$ 16 & 264 & 259 & 530  & 530 & 868 & 872 \\
\hline
\textit{baseline} & \multicolumn{2}{c}{} & \multicolumn{2}{c}{\textit{88}} & \multicolumn{2}{c|}{}\\
\hline

\end{tabular}
}
\caption{
\textbf{Number of labels generated per \personalization (Adult)}. Results computed over five voting processes for different budget distributions for \underline{U}psampling and \underline{W}eighting.
Non-\personalized experiments with the lower group's privacy budget $\log 2$ serve as baselines.
}
\label{tab:n_labels_adult}
\end{table*}

% \begin{table}[ht]
% \scriptsize
% \centering
% {
% \resizebox{\textwidth}{!}{%
% \begin{tabular}{|c|rrr|rrr|rrr|rrr|}
% \hline
% higher & \multicolumn{3}{c|}{25\% ratio} & \multicolumn{3}{c|}{50\% ratio} & \multicolumn{3}{c|}{75\% ratio} & \multicolumn{3}{c|}{100\% ratio}\\
% budget in $\varepsilon$ & accuracy & precision & recall & accuracy & precision & recall & accuracy & precision & recall & accuracy & precision & recall\\
% \hline
% $\log$ 4 & 86.68 & 65.92 & 56.05 & 86.61 & 63.12 & 61.70 & 86.19 & 60.85 & 66.24 & 85.64 & 58.95 & 69.68\\
% \hline
% $\log$ 8 & 86.60 & 63.14 & 61.73 & 85.64 & 58.94 & 69.84 & 84.12 & 55.90 & 75.46 & 82.57 & 53.53 & 79.44\\
% \hline
% $\log$ 16 & 86.25 & 60.81 & 66.12 & 84.12 & 55.84 & 75.48 & 81.82 & 52.53 & 81.01 & 79.52 & 49.94 & 84.79\\
% \hline
% \color{lightgray} baselines &&& \multicolumn{2}{c}{\color{lightgray} accuracy 86.18} && \multicolumn{2}{c}{\color{lightgray} precision 69.41} && \multicolumn{2}{c}{\color{lightgray} recall 48.82} &&\\
% \hline
% \end{tabular}
% }
% }
% \caption{
% \textbf{Voting Accuracy, Teacher Precision, and Teacher Recall (in $\%$) per Skew \Personalization (Adult)}.
% Average voting accuracy, teacher precision, and teacher recall over five voting processes by ten teacher ensembles using \textbf{upsampling} each for different budget distributions (higher budget \& its ratio among minority data).
% Non-\personalized experiments serve as baseline.
% The best voting accuracy is achieved at the higher budget $\log 4$ with $25\%$ ratio.
% \label{tab:voting_acc_prec_rec_skew}
% }
% \end{table}

\begin{table*}[ht]
\centering
{
\footnotesize
% \begin{tabular}{|c|c|rrr|rrr|rrr|c|}
% \hline
% higher & \multicolumn{4}{c}{} & \multicolumn{3}{c}{accuracy ($\%$)} & \multicolumn{4}{c|}{}\\
% \cline{2-12}
% budget && \multicolumn{3}{c|}{25\% ratio} & \multicolumn{3}{c|}{50\% ratio} &
% \multicolumn{3}{c|}{75\% ratio} &\\
% in $\varepsilon$ && u & v & w & u & v & w & u & v & w &\\
% \color{lightgray} baselines & \color{lightgray} min & \multicolumn{3}{c|}{\color{lightgray} average} & \multicolumn{3}{c|}{\color{lightgray} average} & \multicolumn{3}{c|}{\color{lightgray} average} & \color{lightgray} max\\
% \hline
% $\log$ 4 && 88.68 & 60.18 & 88.21 & 90.44 & 78.83 & 90.78 & 91.94 & 88.58 & 91.70 &\\
% & \color{lightgray} 84.51 & \multicolumn{3}{c|}{\color{lightgray} 89.10} & \multicolumn{3}{c|}{\color{lightgray} 91.17} & \multicolumn{3}{c|}{\color{lightgray} 91.81} & \color{lightgray} 92.40\\
% \hline
% $\log$ 8 && 90.56 & 0 & 90.40 & 92.44 & 82.42 & 92.34 & 93.30 & 90.84 & 93.39 &\\
% & \color{lightgray} 84.51 & \multicolumn{3}{c|}{\color{lightgray} 91.81} & \multicolumn{3}{c|}{\color{lightgray} 93.06} & \multicolumn{3}{c|}{\color{lightgray} 93.71} & \color{lightgray} 94.01\\
% \hline
% $\log$ 16 && 91.57 & 0 & 91.39 & 93.56 & 85.85 & 93.27 & 94.31 & 92.78 & 94.28 &\\
% & \color{lightgray} 84.51 & \multicolumn{3}{c|}{\color{lightgray} 93.53} & \multicolumn{3}{c|}{\color{lightgray} 94.09} & \multicolumn{3}{c|}{\color{lightgray} 94.60} & \color{lightgray} 94.76\\
% \hline
% \end{tabular}

\begin{tabular}{|c|rr|rr|rr|}
%\toprule
\hline
higher & \multicolumn{2}{c|}{25\% ratio} & \multicolumn{2}{c|}{50\% ratio} & \multicolumn{2}{c|}{75\% ratio}\\
\cline{2-7}
budget in $\varepsilon$ & U  & W & U  & W & U  & W \\
\hline
$\log$ 4 & 92.32 & 93.26 & 92.52 & 93.08 & 94.38 & 93.70  \\
\hline
$\log$ 8 & 93.12 & 88.94 & 94.48  & 94.68 & 96.20 & 95.74 \\
\hline
$\log$ 16 & 93.96 & 90.24 & 96.38   & 96.32  & 96.90 & 96.60 \\
\hline
\textit{baseline} & \multicolumn{2}{c}{} & \multicolumn{2}{c}{\textit{88.70}} & \multicolumn{2}{c|}{}\\
\hline
\end{tabular}
}
\caption{
\textbf{Student accuracy per \personalization (MNIST)}.
Results for \underline{U}psampling and \underline{W}eighting based on the generated labels (see \Cref{tab:n_labels_mnist_short}).
Non-\personalized experiments with the lower group's privacy budget $\log 2$ serve as a baseline.}
\label{tab:accuracies_mnist}
\end{table*}

\begin{table*}[ht]
\footnotesize
\centering
{
% \begin{tabular}{|c|c|rrr|rrr|rrr|c|}
% \hline
% higher & \multicolumn{4}{c}{} & \multicolumn{3}{c}{accuracy ($\%$)} & \multicolumn{4}{c|}{}\\
% \cline{2-12}
% budget && \multicolumn{3}{c|}{25\% ratio} & \multicolumn{3}{c|}{50\% ratio} &
% \multicolumn{3}{c|}{75\% ratio} &\\
% in $\varepsilon$ && u & v & w & u & v & w & u & v & w &\\
% \color{lightgray} baselines & \color{lightgray} min & \multicolumn{3}{c|}{\color{lightgray} average} & \multicolumn{3}{c|}{\color{lightgray} average} & \multicolumn{3}{c|}{\color{lightgray} average} & \color{lightgray} max\\
% \hline
% $\log$ 4 && 81.02 & 76.17 & 80.87 & 81.76 & 78.62 & 81.76 & 82.16 & 80.70 & 82.26 &\\
% & \color{lightgray} 79.85 & \multicolumn{3}{c|}{\color{lightgray} 81.18} & \multicolumn{3}{c|}{\color{lightgray} 82.00} & \multicolumn{3}{c|}{\color{lightgray} 82.36} & \color{lightgray} 82.52\\
% \hline
% $\log$ 8 && 81.79 & 75.73 & 81.67 & 82.52 & 79.32 & 82.60 & 82.87 & 82.09 & 82.89 &\\
% & \color{lightgray} 79.85 & \multicolumn{3}{c|}{\color{lightgray} 82.36} & \multicolumn{3}{c|}{\color{lightgray} 82.81} & \multicolumn{3}{c|}{\color{lightgray} 83.01} & \color{lightgray} 83.02\\
% \hline
% $\log$ 16 && 82.30 & 75.31 & 82.25 & 82.82 & 79.97 & 82.84 & 83.07 & 82.72 & 83.04 &\\
% & \color{lightgray} 79.85 & \multicolumn{3}{c|}{\color{lightgray} 82.84} & \multicolumn{3}{c|}{\color{lightgray} 83.10} & \multicolumn{3}{c|}{\color{lightgray} 83.19} & \color{lightgray} 83.23\\
% \hline
% \end{tabular}

\begin{tabular}{|c|rr|rr|rr|}
%\toprule
\hline
higher & \multicolumn{2}{c|}{25\% ratio} & \multicolumn{2}{c|}{50\% ratio} & \multicolumn{2}{c|}{75\% ratio}\\
\cline{2-7}
budget in $\varepsilon$ & U  & W & U  & W & U  & W \\
\hline
$\log$ 4 & 81.02  & 80.87 & 81.76  & 81.76 & 82.16  & 82.26  \\
\hline
$\log$ 8 & 81.79   & 81.67 & 82.52 &  82.60 & 82.87 &  82.89 \\
\hline
$\log$ 16 & 82.30  & 82.25 & 82.82  & 82.84  & 83.07 &  83.04 \\
\hline
\textit{baseline} & \multicolumn{2}{c}{} &  \multicolumn{2}{c}{\textit{79.85}} & \multicolumn{2}{c|}{}\\
\hline
\end{tabular}
}
\caption{
\textbf{Student accuracy per \personalization (Adult)}.
Results depict the average accuracies, computed over five voting processes for different budget distributions for \underline{U}psampling and \underline{W}eighting.
Non-\personalized experiments with the lower group's privacy budget $\log 2$ serve as a baseline.
}
\label{tab:accuracies_adult}
\end{table*}

\begin{table*}[ht]
\footnotesize
%\centering
%\resizebox{\textwidth}{!}{%
%\new{
\begin{subtable}[t]{0.48\textwidth}
\begin{tabular}{|c|rr|rr|rr|rr|}
\hline
higher & \multicolumn{2}{c|}{25\% ratio} & \multicolumn{2}{c|}{50\% ratio} & \multicolumn{2}{c|}{75\% ratio} & \multicolumn{2}{c|}{100\% ratio}\\
budget in $\varepsilon$ & low & high & low & high & low & high & low & high\\
\hline
$\log$ 4 & 90.08 & 56.05 & 87.75 & 61.70 & 85.58 & 66.24 & 83.62 & 69.68\\
\hline
$\log$ 8 & 87.74 & 61.73 & 83.57 & 69.84 & 79.99 & 75.46 & 76.86 & 79.44\\
\hline
$\log$ 16 &  85.58 & 66.12 & 79.93 & 75.48 & 75.48 & 81.01 & 71.57 &  84.79\\
\hline
\textit{baseline} & \multicolumn{8}{c|}{(92.55, 48.82)} \\
%\color{lightgray} baselines &&& \multicolumn{2}{c}{\color{lightgray} accuracy 86.18} && \multicolumn{2}{c}{\color{lightgray} precision 69.41} && \multicolumn{2}{c}{\color{lightgray} recall 48.82} &&\\
\hline
\end{tabular}
%}
\caption{
\new{\textbf{Teacher Accuracy.}
}
\label{tab:teacher_acc_skew}
} %end of new
\end{subtable}
\hfill
\begin{subtable}[t]{0.48\textwidth}
\begin{tabular}{|c|rr|rr|rr|rr|}
\hline
higher & \multicolumn{2}{c|}{25\% ratio} & \multicolumn{2}{c|}{50\% ratio} & \multicolumn{2}{c|}{75\% ratio} & \multicolumn{2}{c|}{100\% ratio}\\
budget in $\varepsilon$ & low & high & low & high & low & high & low & high\\
\hline
$\log$ 4 & 95.21 & 55.99 & 93.02 & 64.11 & 91.24 & 68.93 & 88.59 & 74.11\\
\hline
$\log$ 8 & 93.19 & 64.11 & 88.71 & 73.91 & 84.42 & 81.14 & 81.41 & 85.23\\
\hline
$\log$ 16 & 91.02 & 68.55 & 84.47 & 81.36 & 79.99 & 86.49 & 75.32 &  90.22\\
\hline
\textit{baseline} & \multicolumn{8}{c|}{(97.13, 46.29)} \\
%\color{lightgray} baselines &&& \multicolumn{2}{c}{\color{lightgray} accuracy 86.18} && \multicolumn{2}{c}{\color{lightgray} precision 69.41} && \multicolumn{2}{c}{\color{lightgray} recall 48.82} &&\\
\hline
\end{tabular}
%}
\caption{
\new{\textbf{Voting Accuracy.}
}
\label{tab:voting_acc_skew}
} %end of new
\end{subtable}
\caption{
\new{\textbf{Per-group teacher and voting accuracy (in $\%$) for unbalanced \personalization (Adult).} Results reported for \emph{low}-income (majority) and \emph{high}-income (underrepresented) class as an average over 50 different students trained through five voting processes by ten teacher ensembles using \textbf{upsampling}.
The ratios specify what proportion of the underrepresented class was assigned the higher privacy budget. 
The remaining data obtained a privacy budget of $\log 2$.
Non-\personalized experiments with all data points having a privacy budget of $\log 2$ serve as the baseline.
As the proportion of data points from the  underrepresented class that obtain a higher privacy budget (or their respective budget) increases, we observe an increase in accuracy on this class.
At the same time, the accuracy on the majority class decreases.
}}
%}
\end{table*}

\begin{table*}[hb]
\centering
\footnotesize
\begin{tabular}{|c|r|r|r|r|}
\hline
higher & \multicolumn{4}{c|}{\# produced labels}\\
\cline{2-5}
budget in $\varepsilon$ & 25\% ratio & 50\% ratio & 75\% ratio & 100\% ratio \\
%\color{lightgray} baselines & \color{lightgray} average & \color{lightgray} average & \color{lightgray} average & \color{lightgray} average \\
\hline
$\log$ 4 & 90 & 95 & 101 & 109 \\
%& \color{lightgray} 88 & \color{lightgray} 106 & \color{lightgray} 123 & \color{lightgray} 140 & \color{lightgray} 157 & \color{lightgray} 354\\
\hline
$\log$ 8 & 93 & 108 & 132 & 162 \\
%& \color{lightgray} 88 & \color{lightgray} 140 & \color{lightgray} 193 & \color{lightgray} 241 & \color{lightgray} 289 & \color{lightgray} 763\\
\hline
$\log$ 16 & 96 & 129 & 172 & 225 \\
%& \color{lightgray} 88 & \color{lightgray} 209 & \color{lightgray} 320 & \color{lightgray} 428 & \color{lightgray} 524 & \color{lightgray} 1,288\\
\hline
\textit{baseline} & \multicolumn{4}{c|}{\textit{88}} \\
\hline
\end{tabular}
\caption{\textbf{Labels generated per unbalanced \personalization (Adult).} Results depict the average over five voting processes by ten teacher ensembles using upsampling. Ratios indicate the proportion of the underrepresented class with the indicated higher budgets. 
The remaining data receives a privacy budget of $\log 2$.
Non-\personalized experiments with a uniform privacy budget of $\log 2$  serve as a baseline.}
\label{tab:n_labels_skew}
\end{table*}

\end{document}